\documentclass[journaldraft]{IEEEtran}

\usepackage{ifpdf}
\ifpdf
    \usepackage[pdftex]{graphicx}
    \graphicspath{{figs/}{matlab/}}   
    \usepackage[update]{epstopdf}
\else
	\usepackage{graphicx}
	\graphicspath{{figs/}{matlab/}}   
\fi

\usepackage{array}
\usepackage{amsmath}
\usepackage{amsfonts}
\usepackage{amssymb}
\usepackage{amstext}
\usepackage{latexsym}
\usepackage{color}
\usepackage{cite,url}
\usepackage{setspace}
\usepackage{bbm}

\newtheorem{thm}{Theorem}
\newtheorem{prop}{Proposition}
\newtheorem{lemma}{Lemma}
\newtheorem{claim}{Claim}

\usepackage{pdflscape}
\allowdisplaybreaks
\usepackage{multirow}

\DeclareMathSizes{10}{9}{8}{7}

\begin{document}

\title{Detection of Cooperative Interactions in Logistic Regression Models}

\author{Easton~Li~Xu,~\IEEEmembership{Member,~IEEE,} Xiaoning~Qian,~\IEEEmembership{Member,~IEEE,} Tie~Liu,~\IEEEmembership{Senior Member,~IEEE,}
and~Shuguang~Cui,~\IEEEmembership{Fellow,~IEEE}\thanks{\mbox{The work of E.~L.~Xu, X.~Qian, T.~Liu, and S. Cui was supported} in part by grant NSFC-61328102/61629101, by DoD with grant HDTRA1-13-1-0029, and by NSF with grants CNS-1265227, ECCS-1305979, CNS-1343155, CCF-1447235, ECCS-1508051, AST-1547436, and IOS-1547557.}\thanks{E.~L.~Xu, X.~Qian, and T.~Liu are with \mbox{the Department of Electrical and} Computer Engineering, Texas A\&M University, College Station, TX 77843, USA (e-mails: xulimc@gmail.com, xqian@tamu.edu, \mbox{tieliu}@tamu.edu), and S.~Cui is with the Department of Electrical and Computer Engineering, University of California, Davis, CA 95616, USA (e-mail: \mbox{sgcui}@ucdavis.edu).}}

\maketitle

\begin{abstract}
An important problem in the field of bioinformatics is to identify interactive effects among profiled variables for outcome prediction. In this paper, a logistic regression model with pairwise interactions among a set of binary covariates is considered. Modeling the structure of the interactions by a graph, our goal is to recover the interaction graph from independently identically distributed (i.i.d.) samples of the covariates and the outcome.

When viewed as a feature selection problem, a simple quantity called influence is proposed as a measure of the marginal effects of the interaction terms on the outcome. For the case when the underlying interaction graph is known to be acyclic, it is shown that a simple algorithm that is based on a maximum-weight spanning tree with respect to the plug-in estimates of the influences not only has strong theoretical performance guarantees, but can also outperform generic feature selection algorithms for recovering the interaction graph from i.i.d. samples of the covariates and the outcome. Our results can also be extended to the model that includes both individual effects and pairwise interactions via the help of an auxiliary covariate.
\end{abstract}


%

\section{Introduction}
Consider a regression problem with $d$ independent covariates $\mathsf{x}_1,\ldots,\mathsf{x}_d$ and a binary outcome variable $\mathsf{y}$. The covariates are assumed to be uniformly distributed over $\{+1,-1\}$, and the conditional probabilities of the outcome given the covariates are assumed to be logistic:
\begin{align}
\mathrm{Pr}\left(\mathsf{y}=+1| \mathsf{x}_V=x_V\right) & =\sigma\Big(\sum_{\{i,j\}\in E}\beta_{\{i,j\}}x_ix_j\Big),\label{conditional1}\\
\mathrm{Pr}\left(\mathsf{y}=-1| \mathsf{x}_V=x_V\right) & =\sigma\Big(-\sum_{\{i,j\}\in E}\beta_{\{i,j\}}x_ix_j\Big)\label{conditional2}
\end{align}
for some real constants $\beta_{\{i,j\}}$'s, where $V:=\{1,\ldots,d\}$, $E:=\{\{i,j\}:i,j\in V, i \neq j\}$, $\mathsf{x}_V:=(\mathsf{x}_i:i\in V)$, and $\sigma(x):=e^{x}/(1+e^x)$ is the sigmoid function. It is straightforward to verify that $\sigma(x)+\sigma(-x)=1$ for any $x \in \mathbb{R}$, so we have $\mathrm{Pr}\left(\mathsf{y}=+1| \mathsf{x}_V=x_V\right)+\mathrm{Pr}\left(\mathsf{y}=-1| \mathsf{x}_V=x_V\right)=1$ for any $x_V\in \{+1,-1\}^d$.

For any two distinct $i,j\in V$, we say that the covariates $\mathsf{x}_i$ and $\mathsf{x}_j$ {\em interact} if and only if $\beta_{\{i,j\}} \neq 0$. Let $G=(V,I)$ be a simple graph with the vertex set $V$ and edge set $I:=\{\{i,j\} \in E: \beta_{\{i,j\}}\neq 0\}$. Then $G$ captures all pairwise interactions between the covariates in determining the odds of the outcome of interest. Our goal is to recover the graph $G$ from independently identically distributed (i.i.d.) samples of $(\mathsf{x}_V,\mathsf{y})$.

Our main motivation for considering the above pairwise interaction problem is from computational biology, where each covariate $\mathsf{x}_i$ represents the expression of a biomarker (a gene or an environmental factor), and the variable $\mathsf{y}$ represents the phenotypic outcome with respect to a specific phenotype. In computational biology, many complex diseases, such as cancer and diabetes, are conjectured to have complicated underlying disease mechanisms~\cite{mo2003,an2007,ed2010,ch2007,ho2011,adl2013,sa2015}. Multiple candidate risk factors, either genetic or environmental, and their interactions are known to play critical roles in triggering and determining the development of a large family of diseases~\cite{mo2003,an2007,ed2010,ch2007,ho2011,adl2013,sa2015}. Identifying interactive effects among profiled variables not only helps with more accurate identification of critical risk factors for outcome prediction, but also helps reveal functional interactions and understand aberrant system changes that are specifically related to the outcome for effective systems intervention. Our model, of course, is a simplified one from real-world situations, and is studied here since it captures some essential features of the problem (as we shall see shortly) while being relatively simple.

Note that if we let
\begin{align}
\mathsf{z}_{\{i,j\}}:=\mathsf{x}_i\mathsf{x}_j, \quad \forall \{i,j\}\in E
\end{align}
and consider $\mathsf{z}_{\{i,j\}}$'s (instead of $\mathsf{x}_i$'s) as the covariates, then the problem of recovering the graph $G$ can be viewed as a {\em feature selection} problem in statistics and machine learning. In \cite{sa2007}, a basic approach for feature selection is to first use Shannon's mutual information \cite{co2006} to measure the {\em marginal} effects of the covariates on the outcome, and then select the features based on the ranking of the mutual information. More advanced approaches such as the immensely popular mRMR method \cite{pe2005} make incremental selections while taking into account both the relevance to the outcome and the redundancy among the selected features. However, even though Shannon's mutual information provides a compact and model-free measure of correlation between the covariates and the outcome, which is well accepted in the statistics and computational biology communities, it is a complex function of the underlying joint distribution and hence difficult to analyze and estimate from limited data samples. As a result, when applied to specific regression models, the performance of the generic feature selection algorithms is usually difficult to characterize.

Motivated by the recent progress on learning Ising models over arbitrary graphs \cite{br2005}, in this paper we propose a quantity called {\em influence} as a measure of the marginal effects of $\mathsf{z}_{\{i,j\}}$'s on the outcome $\mathsf{y}$. Compared with Shannon's mutual information, influence is a simple function of the low-order joint probabilities between $\mathsf{x}_i$'s and the outcome $\mathsf{y}$, and hence is much easier to analyze and estimate. When the underlying graph $G$ is known to be {\em acyclic}, we show that, a simple algorithm that is based on a maximum-weight spanning tree with respect to the ``plug-in" estimate of the influences and followed by simple thresholding operations, not only has strong theoretical performance guarantees, but can also outperform the generic feature selection algorithms for recovering $G$ from i.i.d. samples of $(\mathsf{x}_V,\mathsf{y})$.

The rest of the paper is organized as follows. In Section~\ref{sec:Id}, we show that any acyclic $G$ can be identified from the influences of $\mathsf{z}_{\{i,j\}}$'s on the outcome $\mathsf{y}$. Building on the results from Section~\ref{sec:Id}, in Section~\ref{sec:Det} we show that any acyclic $G$ can be recovered with probability at least $1-\epsilon$ from $n$ i.i.d. samples of $(\mathsf{x}_V,\mathsf{y})$, where $n=\Theta\left(d\log(d^2/\epsilon)\right)$. In Section~\ref{sec:Ext}, we extend our results of the above sections to the model involving both individual effects and cooperative interactions. In Section~\ref{sec:Sim}, we use computer simulations to demonstrate that the proposed algorithm can outperform the generic feature selection algorithms. Finally, in Section~\ref{sec:Con} we conclude the paper with some remarks.

{\em Notation}. Random variables are written in serif font, and sets are written in capital letters.

\section{Identification of Cooperative Interactions from Low-Order Joint Probabilities}\label{sec:Id}
Our main result in this section is to show that any {\em acyclic} $G$ can be identified from the low-order joint probabilities $(p(\mathsf{x}_i,\mathsf{x}_j,\mathsf{y}):\{i,j\}\in E)$. Towards this goal, let $w$ be a weight assignment over $E$ given by:
\begin{align}
w_{\{i,j\}} :&= \left|\mathrm{Pr}\left(\mathsf{y}=+1|\mathsf{x}_i=+1,\mathsf{x}_j=+1\right)-\right.\nonumber\\
& \hspace{40pt} \left.\mathrm{Pr}\left(\mathsf{y}=-1|\mathsf{x}_i=+1,\mathsf{x}_j=+1\right)\right|\label{eq:W}\\
& = \left|2\mathrm{Pr}\left(\mathsf{y}=+1|\mathsf{x}_i=+1,\mathsf{x}_j=+1\right)-1\right|\label{eq:W2}\\
&=\left|8\mathrm{Pr}\left(\mathsf{x}_i=+1,\mathsf{x}_j=+1,\mathsf{y}=+1\right)-1\right|\label{eq:W3}
\end{align}
for any $\{i,j\}\in E$. Here, \eqref{eq:W2} follows from the fact that
\begin{align*}
\mathrm{Pr}&\left(\mathsf{y}=+1|\mathsf{x}_i=+1,\mathsf{x}_j=+1\right)+\\
&\mathrm{Pr}\left(\mathsf{y}=-1|\mathsf{x}_i=+1,\mathsf{x}_j=+1\right)=1
\end{align*}
and \eqref{eq:W3} is due to the fact that $\mathrm{Pr}(\mathsf{x}_i=+1,\mathsf{x}_j=+1)=1/4$.

The following proposition helps to clarify the meaning of the weight assignment $w$ as defined in \eqref{eq:W}.

\begin{prop}[Influence]
\label{prop:Inf}
Assume that $G=(V,I)$ is acyclic. We have
\begin{align}
w_{\{i,j\}} &= \left|\mathrm{Pr}\left(\mathsf{y}=+1|\mathsf{x}_i=+1,\mathsf{x}_j=+1\right)-\right.\nonumber\\
& \hspace{40pt} \left.\mathrm{Pr}\left(\mathsf{y}=+1|\mathsf{x}_i=+1,\mathsf{x}_j=-1\right)\right|\label{eq:Inf}
\end{align}
for any $\{i,j\}\in E$.
\end{prop}

\begin{proof}
See Section~\ref{pf:Inf}.
\end{proof}

By \eqref{eq:Inf}, $w_{\{i,j\}}$ indicates whether the product $\mathsf{z}_{\{i,j\}}=\mathsf{x}_i\mathsf{x}_j$ has any {\em influence} on the event $\mathsf{y}=+1$ and hence can be a useful indication on whether $\{i,j\} \in I$. This intuition is partially justified by the following proposition.

\begin{prop}[Direct influence]
\label{prop:DI}
Assume that $G=(V,I)$ is acyclic. We have $w_{\{i,j\}}>0$ for any $\{i,j\} \in I$.
\end{prop}

\begin{proof}
See Section~\ref{pf:DI}.
\end{proof}

We say that the product $\mathsf{z}_{\{i,j\}}=\mathsf{x}_i\mathsf{x}_j$ has a {\em direct influence} on the outcome $\mathsf{y}$ if $\{i,j\}\in I$. The above proposition guarantees that direct influences are strictly positive when $G$ is acyclic. The following proposition provides a {\em partial} converse to Proposition~\ref{prop:DI}.

\begin{prop}[Zero influence]
\label{prop:ZI}
Assume that $G=(V,I)$ is acyclic. Then for any two distinct $i,j\in V$, we have $w_{\{i,j\}}=0$ if $i$ and $j$ are disconnected in $G$, or the unique path between $i$ and $j$ in $G$ has an even length.
\end{prop}

\begin{proof}
See Section~\ref{pf:ZI}.
\end{proof}

\begin{thm}[Union of stars]
Assume that each connected component of $G=(V,I)$ is a star. Then for any two distinct $i,j \in V$, we have $\{i,j\} \in I$ if and only if $w_{\{i,j\}}>0$.
\end{thm}

\begin{proof}
This follows immediately from Propositions~\ref{prop:DI} and \ref{prop:ZI} and the fact that if each connected component of $G$ is a star (which implies that $G$ is acyclic), then any two distinct $i,j\in V$ such that $\{i,j\}\not\in I$ must be either disconnected (if they belong to two different connected components) or connected by a unique path of length two (if they belong to the same connected component) in $G$.
\end{proof}

The following example, however, shows that the converse of Proposition~\ref{prop:DI} is {\em not} true in general. Consider $d=4$, $I=\{\{1,2\},\{2,3\},\{3,4\}\}$, and $\beta_{\{1,2\}}=\beta_{\{2,3\}}=\beta_{\{3,4\}}=1$. Note that the graph $G=(V,I)$ here is acyclic and the unique path between $1$ and $4$ is of length three. It is straightforward to calculate that
\begin{align*}
w_{\{1,4\}}&=\frac{(e-1)^3}{2(e^3+1)}>0,
\end{align*}
even though $\{1,4\} \not\in I$.

For any $\{i,j\}\in E\setminus I$, we say that the product $\mathsf{z}_{\{i,j\}}=\mathsf{x}_i\mathsf{x}_j$ has an {\em indirect influence} on the outcome $\mathsf{y}$ if $w_{\{i,j\}}>0$. Due to the possible existence of indirect influences, unlike the unions of stars, a general acyclic $G$ cannot be recovered via edge-by-edge identifications.

The following proposition, however, shows that indirect influences are {\em locally} dominated by direct influences.

\begin{prop}[Indirect influence]
\label{prop:II}
Assume that $G=(V,I)$ is acyclic, and let $\left\{\{i_1,i_2\},\{i_2,i_3\},\ldots,\{i_m,i_{m+1}\}\right\}$ be a path of length $m \geq 2$ in $G$. Then, we have $w_{\{i_1,i_{m+1}\}}<w_{\{i_s,i_{s+1}\}}$ for any $s\in\{1,\ldots,m\}$.
\end{prop}

\begin{proof}
Note that when $m$ is even, by Propositions~\ref{prop:DI} and \ref{prop:ZI} we have $w_{\{i_1,i_{m+1}\}}=0<w_{\{i_s,i_{s+1}\}}$ for any $s\in\{1,\ldots,m\}$. Therefore, we only need to consider the cases where $m$ is odd, for which the proof can be found in Appendix~\ref{pf:II}.
\end{proof}

Let $D:=\{\{i,j\}\in E: \mbox{$i$ and $j$ are disconnected in $G$}\}$. A weight assignment $u$ over $E$ is said to have {\em strict separation} between $I$ and $D$ if there exists a real constant $\eta\geq0$ such that $u_{\{i,j\}} > \eta \geq u_{\{k,l\}}$ for any $\{i,j\}\in I$ and $\{k,l\}\in D$. The consequence of strict separation and local dominance is summarized in the following proposition.

\begin{prop}\label{prop:Cond}
Assume that $G=(V,I)$ is acyclic, let $u$ be a weight assignment over $E$ satisfying: 1) (strict separation) there exists a real constant $\eta\geq0$ such that $u_{\{i,j\}} > \eta \geq u_{\{k,l\}}$ for any $\{i,j\}\in I$ and $\{k,l\}\in D$; and 2) (local dominance) $u_{\{i_1,i_{m+1}\}}<u_{\{i_s,i_{s+1}\}}$ for any path $\left\{\{i_1,i_2\},\{i_2,i_3\},\ldots,\{i_m,i_{m+1}\}\right\}$ of length $m \geq 2$ in $G$ and any $s\in\{1,\ldots,m\}$. Then for any maximum-weight spanning tree $G'=(V,T)$ with respect to the weight assignment $u$, we have $I=T\cap W$ where $W:=\{\{i,j\}\in E: u_{\{i,j\}} > \eta\}$.
\end{prop}

\begin{proof}
See Section~\ref{pf:Cond}.
\end{proof}

The following theorem is the main result of this section.

\begin{thm}[Acyclic graphs]\label{thm:AG}
Assume that $G=(V,I)$ is acyclic, and let $G'=(V,T)$ be a maximum-weight spanning tree with respect to the weight assignment $w$ as defined in \eqref{eq:W}. Then, for any two distinct $i,j\in V$ we have $\{i,j\}\in I$ if and only if $\{i,j\}\in T$ and $w_{i,j}>0$.
\end{thm}

\begin{proof}
Note that by Propositions~\ref{prop:DI} and \ref{prop:ZI}, we have $w_{\{i,j\}} > 0=w_{\{k,l\}}$ for any $\{i,j\}\in I$ and $\{k,l\}\in D$. By Proposition~\ref{prop:II}, we have $w_{\{i_1,i_{m+1}\}}<w_{\{i_s,i_{s+1}\}}$ for any path $\left\{\{i_1,i_2\},\{i_2,i_3\},\ldots,\{i_m,i_{m+1}\}\right\}$ of length $m \geq 2$ in $G$ and any $s\in\{1,\ldots,m\}$. The theorem thus follows directly from Proposition~\ref{prop:Cond} with $u=w$ and $\eta=0$.
\end{proof}

\section{Detection of Cooperative Interactions from Finite Data Samples}\label{sec:Det}
Let $(\mathsf{x}_V[t],\mathsf{y}[t])$, $t=1,\ldots,n$ be $n$ i.i.d. samples of $(\mathsf{x}_V,\mathsf{y})$. To recover the graph $G$, we shall assign to each $\{i,j\} \in E$ a weight that is based on the {\em empirical} joint probability:
\begin{align}\label{eq:What}
\hat{\mathsf{w}}_{\{i,j\}}&:=\left|\frac{8}{n}\sum_{t=1}^n\mathbbm{1}((\mathsf{x}_i[t],\mathsf{x}_j[t],\mathsf{y}[t])=(+1,+1,+1))-1\right|,
\end{align}
where $\mathbbm{1}(\cdot)$ is the indicator function. By \eqref{eq:W3}, for any $\{i,j\} \in E$ we have $\hat{\mathsf{w}}_{\{i,j\}}$ converging to $w_{\{i,j\}}$ in probability in the limit as $n \rightarrow \infty$. The following simple proposition, which follows directly from the well-known Hoeffding's inequality \cite{ho1963}, provides a bound on the rate at which the weight assignment $\hat{\mathsf{w}}$ converges {\em uniformly} to $w$.

\begin{prop}\label{prop:Chernoff}
For any $\{i,j\}\in E$ and $\eta>0$,
\begin{align}
&\mathrm{Pr}\Big(\left|\hat{\mathsf{w}}_{\{i,j\}}-w_{\{i,j\}}\right|\ge \eta\Big) \leq 2e^{-n\eta^2/32}.
\end{align}
\end{prop}

\begin{proof}
Note that \begin{align*}
\Big|\hat{\mathsf{w}}_{\{i,j\}}&-w_{\{i,j\}}\Big|
\le 8\bigg|\mathrm{Pr}\left(\mathsf{x}_i=+1,\mathsf{x}_j=+1,\mathsf{y}=+1\right)\\
&-\frac{1}{n}\sum_{t=1}^n\mathbbm{1}((\mathsf{x}_i[t],\mathsf{x}_j[t],\mathsf{y}[t])=(+1,+1,+1))\bigg|.
\end{align*}
We then finish the proof by applying the Hoeffding's inequality \cite{ho1963}.
\end{proof}

The following propositions, which are generalizations of Propositions~\ref{prop:DI} and \ref{prop:II} respectively, play a key role in adapting the results of Theorem~\ref{thm:AG} from the weight assignment $w$ to $\hat{\mathsf{w}}$.

\begin{prop}\label{prop:PC1}
Assume that $G=(V,I)$ is acyclic and for any $\{i,j\}\in I$ we have $|\beta_{\{i,j\}}|\in [\lambda,\mu]$ for some $\mu \geq \lambda >0$. Let
\begin{align}
\gamma:=\sqrt{\frac{2}{\pi d}} \left[\sigma(\lambda+3\mu)-\sigma(-\lambda+3\mu)\right]>0.\label{eq:gamma}
\end{align}
We have $w_{\{i,j\}}\geq\gamma$ for any $\{i,j\} \in I$.
\end{prop}

\begin{proof}
See Section~\ref{pf:DI}.
\end{proof}

\begin{prop}\label{prop:PC2}
Assume that $G=(V,I)$ is acyclic and for any $\{i,j\}\in I$ we have $|\beta_{\{i,j\}}|\in [\lambda,\mu]$ for some $\mu \geq \lambda >0$. We have $w_{\{i_1,i_{m+1}\}}\leq w_{\{i_s,i_{s+1}\}}-\gamma$ for any path $\left\{\{i_1,i_2\},\{i_2,i_3\},\ldots,\{i_m,i_{m+1}\}\right\}$ of length $m \geq 2$ in $G$ and any $s\in\{1,\ldots,m\}$, where $\gamma$ is defined as in \eqref{eq:gamma}.
\end{prop}

\begin{proof}
See Section~\ref{pf:II}.
\end{proof}

Given Propositions~\ref{prop:PC1} and \ref{prop:PC2}, it is clear that if the estimation error $\left|\hat{\mathsf{w}}_{\{i,j\}}-w_{\{i,j\}}\right|$ is uniformly bounded by $\gamma/2$, an acyclic $G$ can be recovered from $\{\hat{\mathsf{w}}_{\{i,j\}}:\{i,j\}\in E\}$, similar to that from $\{w_{\{i,j\}}:\{i,j\}\in E\}$.

\begin{thm}\label{thm:Main}
Assume that $G=(V,I)$ is acyclic and for any $\{i,j\}\in I$ we have $|\beta_{\{i,j\}}|\in [\lambda,\mu]$ for some $\mu \geq \lambda >0$. Let $\mathsf{G}'=(V,\mathsf{T})$ be a maximum-weight spanning tree with respect to the weight assignment $\hat{\mathsf{w}}$. If
\begin{align}
\left|\hat{\mathsf{w}}_{\{i,j\}}-w_{\{i,j\}}\right|<\frac{\gamma}{2}, \quad \forall \{i,j\}\in E,\label{eq:assump}
\end{align}
then for any two distinct $i,j\in V$ we have $\{i,j\}\in I$ if and only if $\{i,j\}\in \mathsf{T}$ and $\hat{\mathsf{w}}_{\{i,j\}}>\gamma/2$.
\end{thm}

\begin{proof}
By Proposition~\ref{prop:PC1}, we have $w_{\{i,j\}} \geq \gamma >0= w_{\{k,l\}}$ for any $\{i,j\}\in I$ and $\{k,l\}\in D$. Under assumption \eqref{eq:assump}, this implies that $\hat{\mathsf{w}}_{\{i,j\}} > \gamma/2>\hat{\mathsf{w}}_{\{k,l\}}$ for any $\{i,j\}\in I$ and $\{k,l\}\in D$. By Proposition~\ref{prop:PC2}, we have $w_{\{i_1,i_{m+1}\}}\leq w_{\{i_s,i_{s+1}\}}-\gamma$ for any path $\left\{\{i_1,i_2\},\{i_2,i_3\},\ldots,\{i_m,i_{m+1}\}\right\}$ of length $m \geq 2$ in $G$ and any $s\in\{1,\ldots,m\}$. Under assumption \eqref{eq:assump}, this implies that $\hat{\mathsf{w}}_{\{i_1,i_{m+1}\}}<\hat{\mathsf{w}}_{\{i_s,i_{s+1}\}}$ for any path $\left\{\{i_1,i_2\},\{i_2,i_3\},\ldots,\{i_m,i_{m+1}\}\right\}$ of length $m \geq 2$ in $G$ and any $s\in\{1,\ldots,m\}$. The theorem thus follows directly from Proposition~\ref{prop:Cond} with $u=\hat{\mathsf{w}}$ and $\eta=\gamma/2$.
\end{proof}

We then establish the following algorithm to detect $G$ based on Theorem~\ref{thm:Main}.
\begin{table}[h]
\begin{center}
\begin{tabular}{ lc }
\hline\hline
\textbf{Algorithm 1} \\
\hline
\noindent \textbf{Input:} $(\mathsf{x}^d[t],\mathsf{y}[t])$, $t=1,\ldots,n$ and $(\mu,\lambda)$ such that $\mu \geq \lambda >0$.\\
\noindent \textbf{Output:} $\hat{\mathsf{G}}=(V,\hat{\mathsf{I}})$.\\
1 \ For all $1 \leq i < j \leq d$, compute $\hat{\mathsf{w}}_{\{i,j\}}$ according to \eqref{eq:What}.\\
\ \ \ \hspace{-2pt} Compute $\gamma$ from $(\mu,\lambda)$ according to \eqref{eq:gamma}.\\
2 \ Find a maximum-weight spanning tree $\hat{\mathsf{G}}'=(V,\hat{\mathsf{T}})$ over \\
\ \ \ \hspace{2pt}the vertex set $V$ with respect to the weight assignment \\
\ \ \ \hspace{2pt}$(\hat{\mathsf{w}}_{\{i,j\}}:1 \leq i < j \leq d)$.\\
3 \ Return $\hat{\mathsf{G}}=(V,\hat{\mathsf{I}})$ with $\hat{\mathsf{I}}=\left\{\{i,j\}\in \hat{\mathsf{T}}: \hat{\mathsf{w}}_{\{i,j\}}> \gamma/2\right\}.$\\
\hline\hline
\end{tabular}
\end{center}
\end{table}

The sample complexity of the algorithm is summarized in the following theorem.
\begin{thm}
Assume that $G=(V,I)$ is acyclic and for any $\{i,j\}\in I$ we have $|\beta_{i,j}|\in [\lambda,\mu]$ for some $\mu \geq \lambda >0$. Fix $0<\epsilon<1$ and let $n$ be a positive integer such that
\begin{align}
n & \geq \frac{128}{\gamma^2}\log\frac{d^2}{\epsilon}= \frac{64\pi d}{\left[\sigma(\lambda+3\mu)-\sigma(-\lambda+3\mu)\right]^2}\log\frac{d^2}{\epsilon}\label{samplecomplexity}.
\end{align}
Then with probability at least $1-\epsilon$, the algorithm can successfully detect the graph $G$ from $n$ i.i.d. samples of $(\mathsf{x}^d,\mathsf{y})$.
\end{thm}

\begin{proof}
By Proposition~\ref{prop:Chernoff} and Theorem~\ref{thm:Main}, we have
\begin{align*}
\mathrm{Pr}\left(\hat{\mathsf{G}} = G\right) & \geq \mathrm{Pr}\bigg(\bigcap_{\{i,j\}\in E}\Big\{\left|\hat{\mathsf{w}}_{i,j}-w_{i,j}\right|<\frac{\gamma}{2}\Big\}\bigg)\\
& \geq 1-d(d-1)e^{-\frac{n\gamma^2}{128}} \geq 1-d^2\cdot \frac{\epsilon}{d^2}=1-\epsilon.
\end{align*}
This completes the proof of the theorem.
\end{proof}

\section{Extension to Models with both Individual Effects and Cooperative Interactions}\label{sec:Ext}
In this section, we extend the results of Sections~\ref{sec:Id} and \ref{sec:Det} to the models that include both individual effects and cooperative interactions. More specifically, we shall assume that the conditional probability of the outcome $\mathsf{y}$ given the covariates $\mathsf{x}_V$ are given by:
\begin{align}
\mathrm{Pr}\left(\mathsf{y}=+1| \mathsf{x}_V=x_V\right) & =\sigma\Big(\sum_{i \in V}\beta_i x_i+\sum_{\{i,j\}\in E}\beta_{\{i,j\}}x_ix_j\Big),\\
\mathrm{Pr}\left(\mathsf{y}=-1| \mathsf{x}_V=x_V\right) & =\sigma\Big(-\sum_{i \in V}\beta_i x_i-\sum_{\{i,j\}\in E}\beta_{\{i,j\}}x_ix_j\Big)
\end{align}
for some real constants $\beta_i$'s and $\beta_{\{i,j\}}$'s. For any $i\in V$, we say that the covariate $\mathsf{x}_i$ has an {\em individual effect} on $\mathsf{y}$ if and only if $\beta_i\neq 0$; for any $\{i,j\}\in E$, we say that the covariates $\mathsf{x}_i$ and $\mathsf{x}_j$ {\em interact} if and only if $\beta_{\{i,j\}} \neq 0$. Let $\tilde{V}:=V\cup\{0\}$, $\tilde{E}:=\{\{i,j\}: i,j\in\tilde{V},i \neq j\}$, and $\beta_{\{i,0\}}:=\beta_i$ for all $i\in V$. Then, the structure of the model (including both individual effects and cooperative interactions) is fully captured by the simple graph $\tilde{G}=(\tilde{V},\tilde{I})$, where $\tilde{I}:=\{(i,j)\in \tilde{E}:\beta_{\{i,j\}}\neq 0\}$. As before, our goal is to recover the graph $\tilde{G}$ from i.i.d. samples of $(\mathsf{x}_V,\mathsf{y})$.

Toward this goal, we shall introduce an additional covariate $\mathsf{x}_0$, which we assume to be uniformly over $\{+1,-1\}$ and independent of $\mathsf{x}_V$, and use it to define an {\em auxiliary} model, for which the conditional probabilities of the outcome $\tilde{\mathsf{y}}$ given the covariates $\mathsf{x}_{\tilde{V}}$ are given by:
\begin{align}
\mathrm{Pr}\left(\tilde{\mathsf{y}}=+1| \mathsf{x}_{\tilde{V}}=x_{\tilde{V}}\right) & =\sigma\Big(\sum_{\{i,j\}\in \tilde{E}}\beta_{\{i,j\}}x_ix_j\Big),
\\
\mathrm{Pr}\left(\tilde{\mathsf{y}}=-1| \mathsf{x}_{\tilde{V}}=x_{\tilde{V}}\right) & =\sigma\Big(-\sum_{\{i,j\}\in \tilde{E}}\beta_{\{i,j\}}x_ix_j\Big).
\end{align}
By the results of Section~\ref{sec:Id}, if the underlying graph $\tilde{G}$ is known to be acyclic, it can be recovered from the weight assignment:
\begin{align}
\tilde{w}_{\{i,j\}}:=\left|2\mathrm{Pr}\left(\tilde{\mathsf{y}}=+1|\mathsf{x}_i=+1,\mathsf{x}_j=+1\right)-1\right|
\end{align}
for all $\{i,j\}\in \tilde{E}$.

Note that when $j=0$, we trivially have
\begin{align*}
\mathrm{Pr}&\left(\tilde{\mathsf{y}}=+1|\mathsf{x}_i=+1,\mathsf{x}_0=+1\right)=\mathrm{Pr}\left(\mathsf{y}=+1|\mathsf{x}_i=+1\right),
\end{align*}
so
\begin{align}
\tilde{w}_{\{i,0\}}=\left|2\mathrm{Pr}\left(\mathsf{y}=+1|\mathsf{x}_i=+1\right)-1\right|\label{eq:WW1}
\end{align}
for all $i \in V$. On the other hand, when $i,j \neq 0$ and $i \neq j$, we may write
\begin{align*}
\mathrm{Pr}&\left(\tilde{\mathsf{y}}=+1|\mathsf{x}_i=+1,\mathsf{x}_j=+1\right)\\
&=\frac{1}{2}\mathrm{Pr}\left(\mathsf{\tilde{y}}=+1|\mathsf{x}_i=+1,\mathsf{x}_j=+1,\mathsf{x}_0=+1\right)+\\
&\hspace{40pt}\frac{1}{2}\mathrm{Pr}\left(\mathsf{\tilde{y}}=+1|\mathsf{x}_i=+1,\mathsf{x}_j=+1,\mathsf{x}_0=-1\right)\\
&=\frac{1}{2}\mathrm{Pr}\left(\mathsf{y}=+1|\mathsf{x}_i=+1,\mathsf{x}_j=+1\right)+\\
&\hspace{40pt}\frac{1}{2}\mathrm{Pr}\left(\tilde{\mathsf{y}}=+1|\mathsf{x}_i=+1,\mathsf{x}_j=+1,\mathsf{x}_0=-1\right).
\end{align*}
To proceed, further note that
\begin{align}
&2^{d-2}\mathrm{Pr}\left(\tilde{\mathsf{y}}=+1|\mathsf{x}_i=+1,\mathsf{x}_j=+1,\mathsf{x}_0=-1\right)\nonumber\\
&= \sum_{x_V:(x_i,x_j)=(+1,+1)}\mathrm{Pr}\left(\tilde{\mathsf{y}}=+1|\mathsf{x}_V=x_V,\mathsf{x}_0=-1\right)\nonumber\\
& = \sum_{x_V:(x_i,x_j)=(+1,+1)}\sigma\Big(-\sum_{k \in V}\beta_k x_k+\sum_{\{k,l\}\in E}\beta_{\{k,l\}}x_kx_l\Big)\nonumber\\
& = \sum_{x_V:(x_i,x_j)=(-1,-1)}\sigma\Big(\sum_{k \in V}\beta_k x_k+\sum_{\{k,l\}\in E}\beta_{\{k,l\}}x_kx_l\Big)\label{eq:nC}\\
&= \sum_{x_V:(x_i,x_j)=(-1,-1)}\mathrm{Pr}\left(\mathsf{y}=+1|\mathsf{x}_V=x_V\right)\nonumber\\
&= 2^{d-2}\mathrm{Pr}\left(\mathsf{y}=+1|\mathsf{x}_i=-1,\mathsf{x}_j=-1\right)\nonumber
\end{align}
where \eqref{eq:nC} follows from the simple change of variable $x_V \rightarrow -x_V$. It thus follows that
\begin{align*}
\mathrm{Pr}&\left(\tilde{\mathsf{y}}=+1|\mathsf{x}_i=+1,\mathsf{x}_j=+1,\mathsf{x}_0=-1\right)\\
&\hspace{60pt}=\mathrm{Pr}\left(\mathsf{y}=+1|\mathsf{x}_i=-1,\mathsf{x}_j=-1\right)
\end{align*}
and hence
\begin{align*}
\mathrm{Pr}&\left(\tilde{\mathsf{y}}=+1|\mathsf{x}_i=+1,\mathsf{x}_j=+1\right)\\
&=\frac{1}{2}\mathrm{Pr}\left(\mathsf{y}=+1|\mathsf{x}_i=+1,\mathsf{x}_j=+1\right)+\\
& \hspace{40pt}\frac{1}{2}\mathrm{Pr}\left(\mathsf{y}=+1|\mathsf{x}_i=-1,\mathsf{x}_j=-1\right)
\end{align*}
giving
\begin{align}
\tilde{w}_{\{i,j\}}&=\left|\mathrm{Pr}\left(\mathsf{y}=+1|\mathsf{x}_i=+1,\mathsf{x}_j=+1\right)+\right.\nonumber\\
&\hspace{40pt} \left.\mathrm{Pr}\left(\mathsf{y}=+1|\mathsf{x}_i=-1,\mathsf{x}_j=-1\right)-1\right|\label{eq:WW2}
\end{align}
for any $\{i,j\}\in E$.

Combining \eqref{eq:WW1} and \eqref{eq:WW2}, we conclude that the weight assignment $(\tilde{w}_{\{i,j\}}:\{i,j\}\in \tilde{E})$, and hence any acyclic $\tilde{G}$, can be fully recovered from the low-order joint probabilities $(p(\mathsf{x}_i,\mathsf{x}_j,\mathsf{y}):\{i,j\}\in E)$. The results of Section~\ref{sec:Det} can be extended similarly; the details are omitted.

\section{Simulation Results}\label{sec:Sim}
In our simulations, we randomly generate 5,000 logistic regression models, each including 10 independent binary covariates. For each model that we generate, we randomly choose 5 individual effects and 5 interaction pairs, resulting in an acyclic graph as its underlying structure. The model parameters $\beta_i$ for each individual effect and $\beta_{\{i,j\}}$ for each interaction pair are randomly assigned from a uniform distribution over $[-\mu,-\lambda]\cup[\lambda,\mu]$ with $0< \lambda<\mu$. In Fig.~\ref{fig:parameterranges}, we compare the detection rate of Algorithm~1 for different parameter ranges $(\lambda,\mu)=(0.3,0.5)$, $(0.5,1)$, and $(1,2)$, under the sample sizes of 300, 600, 900, 1,200 and 1,500, respectively. Here, we emphasize that a logistic regression model is correctly detected if and only if all 10 features, including 5 individual effects and 5 interaction pairs, are correctly detected. Clearly, the detection rate increases as the lower and upper bounds of the parameter range increase, and Algorithm~1 can achieve a high detection rate (at least 93\%) with a reasonably large number of training samples (around 1200). Also in~Fig.~\ref{fig:parameterranges}, we plot the fitted curves of the detection rate with respect to the increasing sample size based on the functional form $n \propto \log (1/(1-\textrm{detection rate}))$ that we derived in~(\ref{samplecomplexity}) in Section~\ref{sec:Det}. It is clear that the curves fit very well with the empirical results, thus validating the order-tightness of the lower bound in~(\ref{samplecomplexity}).

Next, we shall compare the performance of our algorithm (Algorithm~1) with three generic feature selection \mbox{algorithms:}  mRMR forward selection \cite{pe2005}, mutual information (MI) ranking \cite{sa2007}, and the problem-specific $L_1$-regularized logistic regression algorithm \cite{leesuin2006,parkthastie2007}. For the \mbox{mRMR} forward selection, MI ranking, and $L_1$-regularized logistic regression algorithms, we shall view each of the single variables $\mathsf{x}_i$'s and the interaction terms $\mathsf{x}_i\mathsf{x}_j$'s as a separate feature, and assume that the number of features to be selected is known. The estimation of mutual information in the mRMR forward selection and MI ranking algorithms is based on \cite{ant2001, pa2003}. For the $L_1$-regularized logistic regression algorithm, we tune the regularization parameter till the desired number of nonzero coefficients is obtained. 

\begin{figure}
\centering
\includegraphics[width=0.42\textwidth]{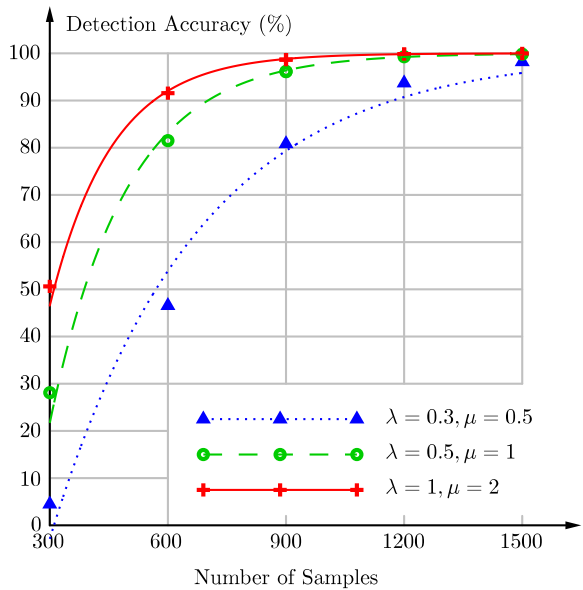}
\caption{The detection rates of Algorithm~1 for different parameter ranges $[-\mu,-\lambda]\cup[\lambda,\mu]$.}
\label{fig:parameterranges}
\end{figure}

In Fig.~\ref{fig:fourmethods-detection}, we compare the detection rate of Algorithm~1 with that of mRMR forward selection, MI ranking, and $L_1$-regularized logistic regression algorithms for $(\lambda,\mu)=(0.3,0.5)$ under a finite number of data samples. As we can see, Algorithm~1 achieves a significantly higher detection rate than the mRMR forward selection and the MI ranking algorithms, especially when the sample size is relatively small. This is due to the facts that: 1) Algorithm 1 exploits the fact that the underlying interaction graph is acyclic while the other two algorithms do not; 2) the proposed influence measure is much easier to estimate than MI. (The performances of the mRMR forward selection and the MI ranking algorithms are nearly identical since the feature candidates are pairwise independent.) On the other hand, the performances of Algorithm~1 and the $L_1$-regularized logistic regression algorithm appear to be very comparable. It is somewhat surprising that the $L_1$-regularized logistic regression algorithm performs well for {\em typical} problem instances with finite sample sizes. Intuitively, this is related to the fact that acyclic graphs are ``sparse" graphs. However, analyzing the sample complexity of the $L_1$-regularized logistic regression algorithm appears to be very challenging.

\begin{figure}
\centering
\includegraphics[width=0.42\textwidth]{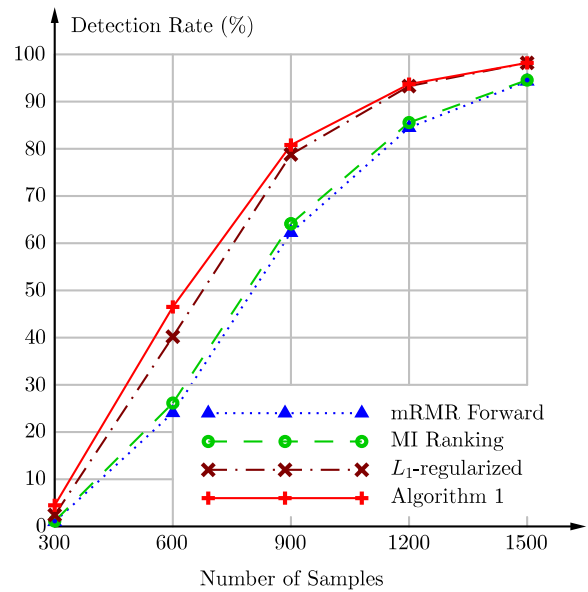}
\caption{Comparison of the detection rates among mRMR forward selection, MI ranking, $L_1$-regularized logistic regression, and Algorithm~1.}
\label{fig:fourmethods-detection}
\end{figure}

For completeness, in Fig.~\ref{fig:fourmethods-falsepositive} and Fig.~\ref{fig:fourmethods-prediction} we also compare the miss detection/false positive rate and prediction accuracy of \mbox{Algorithm~1} with those of the mRMR forward selection, the MI ranking, and the $L_1$-regularized logistic regression algorithms. Note that since the number of features selected is fixed in this case, the miss detection and false positive rates are identical. The prediction accuracy is calculated as follows. For each logistic regression model that we generate, we first obtain the model structure via one of the algorithms under the consideration, followed by standard logistic regression for parameter estimation. Once each logistic regression model is reconstructed, we randomly generate 200 testing samples to estimate the accuracy of the {\em outcome} prediction. As we can see, the relative performance among these algorithms is very similar to the case with the detection rate.

\begin{figure}
\centering
\includegraphics[width=0.42\textwidth]{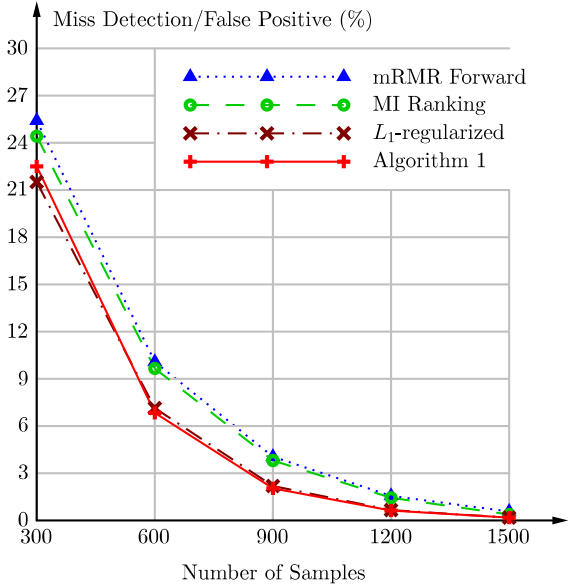}
\caption{Comparison of the false positive rates for detecting nonzero model parameters among mRMR forward selection, MI ranking, $L_1$-regularized logistic regression, and Algorithm~1.}
\label{fig:fourmethods-falsepositive}
\end{figure}

\begin{figure}
\centering
\includegraphics[width=0.42\textwidth]{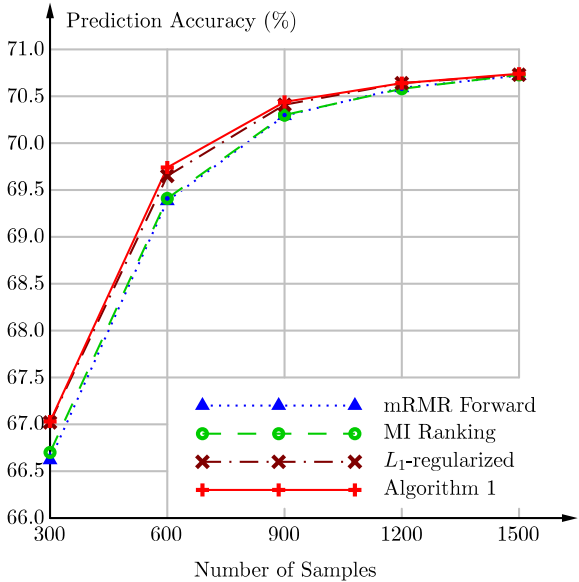}
\caption{Comparison of the prediction accuracy among mRMR forward selection, MI ranking, $L_1$-regularized logistic regression, and Algorithm~1.}
\label{fig:fourmethods-prediction}
\end{figure}

\section{Concluding Remarks}\label{sec:Con}
An important problem in bioinformatics is to identify interactive effects among profiled variables for outcome prediction. In this paper, a simple logistic regression model with pairwise interactions among the binary covariates was considered. Modeling the structure of the interactions by a graph $G$, our goal was to recover $G$ from i.i.d. samples of the covariates and the outcome. When viewed as a feature selection problem, a simple quantity called influence is proposed as a measure of the marginal effects of the interaction terms on the outcome. For the case where $G$ is known to be acyclic, it is shown that a simple algorithm that is based on a maximum-weight spanning tree with respect to the plug-in estimates of the influences not only has strong theoretical performance guarantees, but can also outperform generic feature selection algorithms for recovering the graph from i.i.d. samples of the covariates and the outcome. A sample complexity analysis for detecting the interaction graph was provided, and the results were further extended to the model that includes both individual effects and pairwise interactions.

In our future work, we would like to understand the behavior of the $L_1$-regularized logistic regression algorithm from a theoretical standpoint, and also extend our results to the more challenging case where the interaction graph might be cyclic.

\appendices

\section{Proof of the Propositions}
\subsection{Proof of Proposition~\ref{prop:Inf}}\label{pf:Inf}
We begin with the following lemma.
\begin{lemma}
\label{lemma2}
For any acyclic $G$ and any $i\in V$, we have
\begin{align*}
\mathrm{Pr}(\mathsf{x}_i=+1,\mathsf{y}=+1)=\frac{1}{4}.
\end{align*}
\end{lemma}

A proof of Lemma~\ref{lemma2} can be found in Appendix~\ref{pf:Lemma2}. Now fix $\{i,j\}\in E$. By Lemma~\ref{lemma2}, we have
\begin{align*}
&\mathrm{Pr}(\mathsf{x}_i=+1,\mathsf{x}_j=-1,\mathsf{y}=+1)\\
&=\mathrm{Pr}(\mathsf{x}_i=+1,\mathsf{y}=+1)-\mathrm{Pr}(\mathsf{x}_i=+1,\mathsf{x}_j=+1,\mathsf{y}=+1)\\
&=\frac{1}{4}-\mathrm{Pr}(\mathsf{x}_i=+1,\mathsf{x}_j=+1,\mathsf{y}=+1)\\
&\stackrel{(a)}{=}\mathrm{Pr}(\mathsf{x}_i=+1,\mathsf{x}_j=+1)-\mathrm{Pr}(\mathsf{x}_i=+1,\mathsf{x}_j=+1,\mathsf{y}=+1)\\
&=\mathrm{Pr}(\mathsf{x}_i=+1,\mathsf{x}_j=+1,\mathsf{y}=-1),
\end{align*}
and hence
\begin{align}
&\mathrm{Pr}(\mathsf{y}=+1|\mathsf{x}_i=+1,\mathsf{x}_j=-1)\stackrel{(b)}{=}\mathrm{Pr}(\mathsf{y}=-1|\mathsf{x}_i=+1,\mathsf{x}_j=+1),\label{eq:W4}
\end{align}
where $(a)$ and $(b)$ are due to the fact that $\mathrm{Pr}(\mathsf{x}_i=+1,\mathsf{x}_j=+1)=\mathrm{Pr}(\mathsf{x}_i=+1,\mathsf{x}_j=-1)=1/4$. Substituting \eqref{eq:W4} into \eqref{eq:W} completes the proof of Proposition~\ref{prop:Inf}.

\subsection{Proof of Propositions~\ref{prop:DI} and \ref{prop:PC1}}\label{pf:DI}
Fix $\{i,j\} \in I$. For any $x_V\in\{+1,-1\}^d$, let
\begin{align}
\zeta_{\{i,j\}}^+(x_V)&:=\sigma\Big(\beta_{\{i,j\}}+\sum_{\{k,l\}\in I\setminus\{\{i,j\}\}}\beta_{\{k,l\}}x_kx_l\Big),\\
\zeta_{\{i,j\}}^-(x_V)&:=\sigma\Big(-\beta_{\{i,j\}}+\sum_{\{k,l\}\in I\setminus\{\{i,j\}\}}\beta_{\{k,l\}}x_kx_l\Big),\\
\zeta_{\{i,j\}}(x_V) &:=\zeta_{\{i,j\}}^+(x_V)-\zeta_{\{i,j\}}^-(x_V).
\end{align}
We have the following lemma, for which a proof is provided in Appendix~\ref{pf:Lemma3}.
\begin{lemma}\label{lemma3}
Assume that $G$ is acyclic. We have
\begin{align}
&\mathrm{Pr}(\mathsf{y}=+1|\mathsf{x}_i=+1,\mathsf{x}_j=+1)-\mathrm{Pr}(\mathsf{y}=-1|\mathsf{x}_i=+1,\mathsf{x}_j=+1)\nonumber\\
&=2^{-(d-2)}\sum_{x_V:(x_i,x_j)=(+1,+1)}\zeta_{\{i,j\}}(x_V).\label{eq:Ste}
\end{align}
\end{lemma}

To prove Proposition~\ref{prop:DI}, note that the sigmoid function $\sigma(\cdot)$ is {\em strictly} monotone increasing. We thus have $\zeta_{\{i,j\}}(x_V)> 0$ for any $x_V\in\{+1,-1\}^d$ when $\beta_{\{i,j\}}>0$, and $\zeta_{\{i,j\}}(x_V)< 0$ for any $x_V\in\{+1,-1\}^d$ when $\beta_{\{i,j\}}<0$. It thus follows from Lemma~\ref{lemma3} and \eqref{eq:W} that
\begin{align*}
w_{\{i,j\}} &= 2^{-(d-2)}\left|\sum_{x_V:(x_i,x_j)=(+1,+1)}\zeta_{\{i,j\}}(x_V)\right|>0.
\end{align*}
This completes the proof of Proposition~\ref{prop:DI}.

To prove Proposition~\ref{prop:PC1}, let $G'=(V,T)$ be a tree that covers $G$. With a slight abuse of notation, let
\begin{align*}
&\zeta_{\{i,j\}}(z_{T\setminus\{\{i,j\}\}})\\
&:=\sigma\Big(\beta_{\{i,j\}}+\sum_{\{k,l\}\in T\setminus\{\{i,j\}\}}\beta_{\{k,l\}}z_{\{k,l\}}\Big)-\\
& \hspace{14pt}\sigma\Big(-\beta_{\{i,j\}}+\sum_{\{k,l\}\in T\setminus\{\{i,j\}\}}\beta_{\{k,l\}}z_{\{k,l\}}\Big)\\
&\stackrel{(a)}{=}\sigma\Big(\beta_{\{i,j\}}+\sum_{\{k,l\}\in I\setminus\{\{i,j\}\}}\beta_{\{k,l\}}z_{\{k,l\}}\Big)-\\
& \hspace{14pt}\sigma\Big(-\beta_{\{i,j\}}+\sum_{\{k,l\}\in I\setminus\{\{i,j\}\}}\beta_{\{k,l\}}z_{\{k,l\}}\Big),
\end{align*}
where $(a)$ follows from the fact that $\beta_{\{k,l\}}=0$ for any $\{k,l\}\not\in I$. Note that: 1) $\zeta_{\{i,j\}}(x_V)=\zeta_{\{i,j\}}(z_{T\setminus\{\{i,j\}\}})$ for any $x_V\in\{+1,-1\}^d$; 2) $G'-\{i,j\}$ is a union of two trees where $i$ and $j$ are in different trees, such that the mapping between $x_V$ and $(z_{T\setminus\{\{i,j\}\}},x_i,x_j)$ is {\em one-on-one}. We thus have
\begin{align*}
&\sum_{x_V:(x_i,x_j)=(+1,+1)}\zeta_{\{i,j\}}(x_V)\\
&\hspace{40pt} =\sum_{z_{T\setminus\{\{i,j\}\}}\in \{+1,-1\}^{d-2}}\zeta_{\{i,j\}}(z_{T\setminus\{\{i,j\}\}}),
\end{align*}
and hence
\begin{align*}
&\mathrm{Pr}(\mathsf{y}=+1|\mathsf{x}_i=+1,\mathsf{x}_j=+1)-\mathrm{Pr}(\mathsf{y}=-1|\mathsf{x}_i=+1,\mathsf{x}_j=+1)\nonumber\\
&=2^{-(d-2)}\sum_{z_{T\setminus\{\{i,j\}\}}\in \{+1,-1\}^{d-2}}\zeta_{\{i,j\}}(z_{T\setminus\{\{i,j\}\}}).
\end{align*}

Assume without loss of generality that $\beta_{\{i,j\}}>0$. (Otherwise, we may simply replace $\beta_{\{i,j\}}$ by $-\beta_{\{i,j\}}$, and the rest of the proof remains the same.) By the monotonicity of the sigmoid function $\sigma(\cdot)$, we have $\zeta_{\{i,j\}}(z_{T\setminus\{\{i,j\}\}})> 0$ for any $z_{T\setminus\{\{i,j\}\}}\in\{+1,-1\}^{d-2}$. It thus follows from Lemma~\ref{lemma3} and \eqref{eq:W} that
\begin{align*}
w_{\{i,j\}} &=2^{-(d-2)}\sum_{z_{T\setminus\{\{i,j\}\}}\in \{+1,-1\}^{d-2}}\zeta_{\{i,j\}}(z_{T\setminus\{\{i,j\}\}})\\
&\geq 2^{-(d-2)}\sum_{z_{T\setminus\{\{i,j\}\}}\in \Delta_1}\zeta_{\{i,j\}}(z_{T\setminus\{\{i,j\}\}}),
\end{align*}
where
\begin{align*}
\Delta_1&:=\left\{z_{T\setminus\{\{i,j\}\}}\in \{+1,-1\}^{d-2}:\right.\nonumber\\
&\hspace{60pt}\left.\left|\sum_{\{k,l\}\in T\setminus\{i,j\}}\beta_{\{k,l\}}z_{\{k,l\}}\right|\leq \mu\right\}.
\end{align*}
For any $z_{T\setminus\{\{i,j\}\}}\in \Delta_1$, we have
\begin{align*}
\zeta_{\{i,j\}}&(z_{T\setminus\{\{i,j\}\}})\\
& \geq \min_{|x|\leq 3\mu}\left[\sigma\Big(\beta_{\{i,j\}}+x\Big)-\sigma\Big(-\beta_{\{i,j\}}+x\Big)\right]\\
&= \sigma\Big(\beta_{\{i,j\}}+3\mu\Big)-\sigma\Big(-\beta_{\{i,j\}}+3\mu\Big)\\
& \geq \min_{0<\beta\leq\lambda}\left[\sigma(\beta+3\mu)-\sigma(-\beta+3\mu)\right]\\
&= \sigma(\lambda+3\mu)-\sigma(-\lambda+3\mu).
\end{align*}
It follows that
\begin{align*}
w_{\{i,j\}} &\geq \frac{|\Delta_1|}{2^{d-2}}\left[\sigma(\lambda+3\mu)-\sigma(-\lambda+3\mu)\right]\\
&\geq \sqrt{\frac{2}{\pi d}} \left[\sigma(\lambda+3\mu)-\sigma(-\lambda+3\mu)\right]=\gamma;
\end{align*}
here the last inequality follows from the following lemma, whose proof can be found in Appendix~\ref{pf:Lemma1}.

\begin{lemma}\label{lemma1}
Let $\mathsf{z}_i$, $i=1,\ldots,q$, be i.i.d. random variables, each uniformly distributed over $\{+1,-1\}$. Then, for any real constants $a_i$'s we have
\begin{align}
\mathrm{Pr}\Big(\Big|\sum_{i=1}^{q}a_i\mathsf{x}_i\Big|\leq a\Big) \geq \sqrt{\frac{2}{\pi(q+2)}},\label{eq:lemma}
\end{align}
where $a:=\max(|a_i|:i=1,\ldots,q)$.
\end{lemma}

\subsection{Proof of Proposition~\ref{prop:ZI}}\label{pf:ZI}
Fix $\{i,j\}\in E$, and assume that $i$ and $j$ are either disconnected in $G$, or the length of the unique path between $i$ and  $j$ in $G$ is even. By the assumption that $G$ is acyclic, there must exist a vertex bipartition $(V_1,V_2)$ of $G$ such that both $i$ and $j$ are in $V_1$. Note that
\begin{align*}
I \subseteq \{\{k,l\}:k\in V_1,l\in V_2\};
\end{align*}
thus we have
\begin{align*}
&2^{d-2}\mathrm{Pr}(\mathsf{y}=+1|\mathsf{x}_i=+1,\mathsf{x}_j=+1)\\
&=\sum_{x_V:(x_i,x_j)=(+1,+1)} \mathrm{Pr}\left(\mathsf{y}=+1|\mathsf{x}_V=x_V\right)\\
&=\sum_{x_V:(x_i,x_j)=(+1,+1)} \sigma\Big(\sum_{k\in V_1}\sum_{l\in V_2}\beta_{\{k,l\}}x_kx_l\Big)\\
&\stackrel{(a)}{=}\sum_{x_V:(x_i,x_j)=(+1,+1)} \sigma\Big(-\sum_{k\in V_1}\sum_{l\in V_2}\beta_{\{k,l\}}x_kx_l\Big)\\
&=\sum_{x_V:(x_i,x_j)=(+1,+1)} \mathrm{Pr}\left(\mathsf{y}=-1|\mathsf{x}_V=x_V\right)\\
&=2^{d-2}\mathrm{Pr}(\mathsf{y}=-1|\mathsf{x}_i=+1,\mathsf{x}_j=+1),
\end{align*}
where $(a)$ follows from the change of variable $x_l \rightarrow -x_l$ for all $l \in V_2$. We then have
\begin{align*}
\mathrm{Pr}(\mathsf{y}=+1|\mathsf{x}_i=+1,\mathsf{x}_j=+1)=\mathrm{Pr}(\mathsf{y}=-1|\mathsf{x}_i=+1,\mathsf{x}_j=+1),
\end{align*}
which, together with \eqref{eq:W}, implies that $w_{\{i,j\}}=0$. This completes the proof of Proposition~\ref{prop:ZI}.

\subsection{Proof of Propositions~\ref{prop:II} and \ref{prop:PC2}}\label{pf:II}
As mentioned, we only need to consider $m\geq 2$ that is odd. For notational convenience, let us assume without loss of generality that $i_s=s$ for all $s\in \{1,\ldots,m+1\}$.

To compare $w_{\{1,m+1\}}$ with $w_{\{s,s+1\}}$ for $s\in\{1,\ldots,m\}$, note that by \eqref{eq:W} and the fact that $\mathrm{Pr}(\mathsf{y}=+1|\mathsf{x}_1=+1,\mathsf{x}_{m+1}=+1)+\mathrm{Pr}(\mathsf{y}=-1|\mathsf{x}_1=+1,\mathsf{x}_{m+1}=+1)=1$ we have
\begin{align*}
w&_{\{1,m+1\}}\\
&=\max\left\{2\mathrm{Pr}(\mathsf{y}=+1|\mathsf{x}_1=+1,\mathsf{x}_{m+1}=+1)-1,\right.\\
&\hspace{50pt}\left.2\mathrm{Pr}(\mathsf{y}=-1|\mathsf{x}_1=+1,\mathsf{x}_{m+1}=+1)-1\right\}\\
&=\max\left\{1-2\mathrm{Pr}(\mathsf{y}=+1|\mathsf{x}_1=+1,\mathsf{x}_{m+1}=+1),\right.\\
&\hspace{50pt}\left.1-2\mathrm{Pr}(\mathsf{y}=-1|\mathsf{x}_1=+1,\mathsf{x}_{m+1}=+1)\right\}.
\end{align*}
Below, we shall compare $\mathrm{Pr}(\mathsf{y}=+1|\mathsf{x}_s=+1,\mathsf{x}_{s+1}=+1)$ with $\mathrm{Pr}(\mathsf{y}=+1|\mathsf{x}_1=+1,\mathsf{x}_{m+1}=+1)$ and $\mathrm{Pr}(\mathsf{y}=-1|\mathsf{x}_1=+1,\mathsf{x}_{m+1}=+1)$ for $s=1$, $s=m$, and $s\in\{2,\ldots,m-1\}$ separately.

Case 1: $s=1$. We have the following lemma, for which a proof is provided in Appendix~\ref{pf:Lemma4-5}.
\begin{lemma}\label{lemma4}
For any $m \geq 2$, we have
\begin{align}
&\mathrm{Pr}(\mathsf{y}=+1|\mathsf{x}_1=+1,\mathsf{x}_{2}=+1)-\nonumber\\
& \hspace{36pt}\mathrm{Pr}(\mathsf{y}=+1|\mathsf{x}_1=+1,\mathsf{x}_{m+1}=+1)\nonumber\\
&=2^{-(d-2)}\sum_{x_V:(x_1,x_2,x_{m+1})=(+1,+1,-1)}\zeta_{\{1,2\}}(x_V).\label{eq:Bos1}
\end{align}
Furthermore, for any $m\geq 2$ that is odd, we have
\begin{align}
&\mathrm{Pr}(\mathsf{y}=+1|\mathsf{x}_1=+1,\mathsf{x}_{2}=+1)-\nonumber\\
& \hspace{36pt}\mathrm{Pr}(\mathsf{y}=-1|\mathsf{x}_1=+1,\mathsf{x}_{m+1}=+1)\nonumber\\
&=2^{-(d-2)}\sum_{x_V:(x_1,x_2,x_{m+1})=(+1,+1,+1)}\zeta_{\{1,2\}}(x_V).\label{eq:Bos2}
\end{align}
\end{lemma}

To show $w_{\{1,2\}} > w_{\{1,m+1\}}$, we shall consider the cases with $\beta_{\{1,2\}}>0$ and $\beta_{\{1,2\}}<0$ separately. When $\beta_{\{1,2\}}>0$, we have $\zeta_{\{1,2\}}(x_V)>0$ for any $x_V\in\{+1,-1\}^d$. By Lemma~\ref{lemma4}, we have
\begin{align*}
w_{\{1,2\}}&=|2\mathrm{Pr}(\mathsf{y}=+1|\mathsf{x}_1=+1,\mathsf{x}_{2}=+1)-1|\\
&\geq2\mathrm{Pr}(\mathsf{y}=+1|\mathsf{x}_1=+1,\mathsf{x}_{2}=+1)-1\\
&>\max\left\{2\mathrm{Pr}(\mathsf{y}=+1|\mathsf{x}_1=+1,\mathsf{x}_{m+1}=+1)-1,\right.\\
& \hspace{34pt} \left.2\mathrm{Pr}(\mathsf{y}=-1|\mathsf{x}_1=+1,\mathsf{x}_{m+1}=+1)-1\right\}\\
&=w_{\{1,m+1\}}.
\end{align*}
When $\beta_{\{1,2\}}<0$, we have $\zeta_{\{1,2\}}(x_V)<0$ for any $x_V\in\{+1,-1\}^d$. By Lemma~\ref{lemma4}, we have
\begin{align*}
w_{\{1,2\}}&=|2\mathrm{Pr}(\mathsf{y}=+1|\mathsf{x}_1=+1,\mathsf{x}_{2}=+1)-1|\\
&\geq1-2\mathrm{Pr}(\mathsf{y}=+1|\mathsf{x}_1=+1,\mathsf{x}_{2}=+1)\\
&>\max\left\{1-2\mathrm{Pr}(\mathsf{y}=+1|\mathsf{x}_1=+1,\mathsf{x}_{m+1}=+1),\right.\\
& \hspace{34pt} \left.1-2\mathrm{Pr}(\mathsf{y}=-1|\mathsf{x}_1=+1,\mathsf{x}_{m+1}=+1)\right\}\\
&=w_{\{1,m+1\}}.
\end{align*}

To show $w_{\{1,2\}} \geq w_{\{1,m+1\}}+\gamma$, let $G'=(V,T)$ be a tree that covers $G$. Note that: 1) \begin{align*}
\zeta_{\{1,2\}}(x_V)&=\zeta_{\{1,2\}}\Big(z_{T\setminus\{\{1,2\},\{2,3\}\}},\\
&z_{2,3}=x_2x_{m+1}/\prod_{r=3}^{m}z_{\{r,r+1\}}\Big),
\end{align*}
 for any $x_V\in\{+1,-1\}^d$; 2) $G'-\{\{1,2\},\{2,3\}\}$ is a union of three trees where $1$, $2$, and $m+1$ are in different trees, such that the mapping between $x_V$ and $(z_{T\setminus\{\{1,2\},\{2,3\}\}},x_1,x_2,x_{m+1})$ is {\em one-on-one}. We thus have
\begin{align*}
&\sum_{x_V:(x_1,x_2,x_{m+1})=(+1,+1,a)}\zeta_{\{1,2\}}(x_V)\\
&=\sum_{z_{T\setminus\{\{1,2\},\{2,3\}\}}\in \{+1,-1\}^{d-3}}\zeta_{\{1,2\}}\Big(z_{T\setminus\{\{1,2\},\{2,3\}\}},\\
& \hspace{50pt} z_{\{2,3\}}=a/\prod_{r=3}^{m}z_{\{r,r+1\}}\Big)
\end{align*}
for any $a\in \{+1,-1\}$. It follows that
\begin{align*}
&\mathrm{Pr}(\mathsf{y}=+1|\mathsf{x}_1=+1,\mathsf{x}_2=+1)-\\
&\hspace{50pt}\mathrm{Pr}(\mathsf{y}=+1|\mathsf{x}_1=+1,\mathsf{x}_{m+1}=+1)\nonumber\\
&=2^{-(d-2)}\sum_{z_{T\setminus\{\{1,2\},\{2,3\}\}}\in \{+1,-1\}^{d-3}}\zeta_{\{1,2\}}\Big(\\
& \hspace{40pt} z_{T\setminus\{\{1,2\},\{2,3\}\}},z_{\{2,3\}}=-\prod_{r=3}^{m}z_{\{r,r+1\}}\Big)
\end{align*}
and
\begin{align*}
&\mathrm{Pr}(\mathsf{y}=+1|\mathsf{x}_1=+1,\mathsf{x}_2=+1)-\\
&\hspace{50pt}\mathrm{Pr}(\mathsf{y}=-1|\mathsf{x}_1=+1,\mathsf{x}_{m+1}=+1)\nonumber\\
&=2^{-(d-2)}\sum_{z_{T\setminus\{\{1,2\},\{2,3\}\}}\in \{+1,-1\}^{d-3}}\zeta_{\{1,2\}}\Big(\\
& \hspace{40pt} z_{T\setminus\{\{1,2\},\{2,3\}\}},z_{\{2,3\}}=\prod_{r=3}^{m}z_{\{r,r+1\}}\Big).
\end{align*}

Assume without loss of generality that $\beta_{\{1,2\}}>0$. (Otherwise, we may simply replace $\beta_{\{1,2\}}$ by $-\beta_{\{1,2\}}$, and the rest of the proof remains the same.) By the monotonicity of the sigmoid function $\sigma(\cdot)$, we have $\zeta_{\{1,2\}}(z_{T\setminus\{\{1,2\}\}})> 0$ for any $z_{T\setminus\{\{1,2\}\}}\in\{+1,-1\}^{d-2}$. We thus have
\begin{align*}
&\mathrm{Pr}(\mathsf{y}=+1|\mathsf{x}_1=+1,\mathsf{x}_2=+1)-\\
&\hspace{50pt}\mathrm{Pr}(\mathsf{y}=+1|\mathsf{x}_1=+1,\mathsf{x}_{m+1}=+1)\nonumber\\
&\geq 2^{-(d-2)}\sum_{z_{T\setminus\{\{1,2\},\{2,3\}\}}\in \Delta_2}\zeta_{\{1,2\}}\Big(\\
& \hspace{40pt} z_{T\setminus\{\{1,2\},\{2,3\}\}},z_{\{2,3\}}=-\prod_{r=3}^{m}z_{\{r,r+1\}}\Big),
\end{align*}
where
\begin{align*}
\Delta_2&:=\left\{z_{T\setminus\{\{i,j\}\}}\in \{+1,-1\}^{d-3}:\right.\nonumber\\
&\hspace{60pt}\left.\left|\sum_{\{k,l\}\in T\setminus\{\{1,2\},\{2,3\}\}}\beta_{\{k,l\}}z_{\{k,l\}}\right|\leq \mu\right\}.
\end{align*}
For any $z_{T\setminus\{\{1,2\},\{2,3\}\}}\in \Delta_2$, we have
\begin{align*}
&\left|\beta_{\{2,3\}}\Big(-\prod_{r=3}^{m}z_{\{r,r+1\}}\Big)+\sum_{\{k,l\}\in T\setminus\{\{1,2\},\{2,3\}\}}\beta_{\{k,l\}}z_{\{k,l\}}\right|\\
&\leq \left|\beta_{\{2,3\}}\right|+\left|\sum_{\{k,l\}\in T\setminus\{\{1,2\},\{2,3\}\}}\beta_{\{k,l\}}z_{\{k,l\}}\right|\\
&\leq \mu+\mu\leq 3\mu,
\end{align*}
and hence
\begin{align*}
&\zeta_{\{1,2\}}\Big(z_{T\setminus\{\{1,2\},\{2,3\}\}},z_{\{2,3\}}=-\prod_{r=3}^{m}z_{\{r,r+1\}}\Big)\\
& \geq \min_{|x|\leq 3\mu}\left[\sigma\Big(\beta_{\{1,2\}}+x\Big)-\sigma\Big(-\beta_{\{1,2\}}+x\Big)\right]\\
&= \sigma\Big(\beta_{\{1,2\}}+3\mu\Big)-\sigma\Big(-\beta_{\{1,2\}}+3\mu\Big)\\
& \geq \min_{0<\beta\leq\lambda}\left[\sigma(\beta+3\mu)-\sigma(-\beta+3\mu)\right]\\
&= \sigma(\lambda+3\mu)-\sigma(-\lambda+3\mu).
\end{align*}
It follows from Lemma~\ref{lemma1} that
\begin{align*}
&\mathrm{Pr}(\mathsf{y}=+1|\mathsf{x}_1=+1,\mathsf{x}_2=+1)-\\
& \hspace{50pt}\mathrm{Pr}(\mathsf{y}=+1|\mathsf{x}_1=+1,\mathsf{x}_{m+1}=+1)\nonumber\\
&\geq \frac{|\Delta_2|}{2^{d-2}}\left[\sigma(\lambda+3\mu)-\sigma(-\lambda+3\mu)\right]\\
&\geq
\frac{1}{\sqrt{2\pi d}}\left[\sigma(\lambda+3\mu)-\sigma(-\lambda+3\mu)\right]=\frac{\gamma}{2}.
\end{align*}
Also, by a completely analogous argument, we have
\begin{align*}
&\mathrm{Pr}(\mathsf{y}=+1|\mathsf{x}_1=+1,\mathsf{x}_2=+1)-\\
& \hspace{50pt}\mathrm{Pr}(\mathsf{y}=-1|\mathsf{x}_1=+1,\mathsf{x}_{m+1}=+1)\geq \frac{\gamma}{2},
\end{align*}
and therefore,
$$w_{\{1,2\}}-w_{\{1,m+1\}}\geq 2\cdot \frac{\gamma}{2}=\gamma.$$

Case 2: $s=m$. The proof of this case is completely analogous to the previous case with $s=1$ and is hence omitted here.

Case 3: $s\in\{2,\ldots,m-1\}$. Fix $s$. We have the following lemma, for which a proof is provided in Appendix~\ref{pf:Lemma4-5}.

\begin{lemma}\label{lemma5}
For any $m \geq 3$, we have
\begin{align}
&2^{d-2}\left[\mathrm{Pr}(\mathsf{y}=+1|\mathsf{x}_s=+1,\mathsf{x}_{s+1}=+1)-\right.\nonumber\\
& \hspace{88pt}\left.\mathrm{Pr}(\mathsf{y}=+1|\mathsf{x}_1=+1,\mathsf{x}_{m+1}=+1)\right]\nonumber\\
&=\sum_{x_V:(x_1,x_s,x_{s+1},x_{m+1})=(+1,+1,+1,-1)}\zeta_{\{s,s+1\}}(x_V)+\nonumber\\
&\hspace{13pt}\sum_{x_V:(x_1,x_s,x_{s+1},x_{m+1})=(-1,+1,+1,+1)}\zeta_{\{s,s+1\}}(x_V).\label{eq:TB1}
\end{align}
For any $m \geq 3$ that is odd, we have
\begin{align}
&2^{d-2}\left[\mathrm{Pr}(\mathsf{y}=+1|\mathsf{x}_s=+1,\mathsf{x}_{s+1}=+1)-\right.\nonumber\\
& \hspace{88pt}\left.\mathrm{Pr}(\mathsf{y}=-1|\mathsf{x}_1=+1,\mathsf{x}_{m+1}=+1)\right]\nonumber\\
&=\sum_{x_V:(x_1,x_s,x_{s+1},x_{m+1})=(+1,+1,+1,+1)}\zeta_{\{s,s+1\}}(x_V)+\nonumber\\
&\hspace{13pt}\sum_{x_V:(x_1,x_s,x_{s+1},x_{m+1})=(-1,+1,+1,-1)}\zeta_{\{s,s+1\}}(x_V).\label{eq:TB2}
\end{align}
\end{lemma}

To show $w_{\{s,s+1\}} > w_{\{1,m+1\}}$, we shall consider the cases with $\beta_{\{s,s+1\}}>0$ and $\beta_{\{s,s+1\}}<0$ separately. When $\beta_{\{s,s+1\}}>0$, we have $\zeta_{\{s,s+1\}}(x_V)>0$ for any $x_V\in\{+1,-1\}^d$. By Lemma~\ref{lemma5}, we have
\begin{align*}
w_{\{s,s+1\}}&\geq2\mathrm{Pr}(\mathsf{y}=+1|\mathsf{x}_s=+1,\mathsf{x}_{s+1}=+1)-1\\
&>\max\left\{2\mathrm{Pr}(\mathsf{y}=+1|\mathsf{x}_1=+1,\mathsf{x}_{m+1}=+1)-1,\right.\\
& \hspace{34pt} \left.2\mathrm{Pr}(\mathsf{y}=-1|\mathsf{x}_1=+1,\mathsf{x}_{m+1}=+1)-1\right\}\\
&=w_{\{1,m+1\}}.
\end{align*}
When $\beta_{\{s,s+1\}}<0$, we have $\zeta_{\{s,s+1\}}(x_V)<0$ for any $x_V\in\{+1,-1\}^d$. By Lemma~\ref{lemma5}, we have
\begin{align*}
w_{\{s,s+1\}}&\geq1-2\mathrm{Pr}(\mathsf{y}=+1|\mathsf{x}_s=+1,\mathsf{x}_{s+1}=+1)\\
&>\max\left\{1-2\mathrm{Pr}(\mathsf{y}=+1|\mathsf{x}_1=+1,\mathsf{x}_{m+1}=+1),\right.\\
& \hspace{34pt} \left.1-2\mathrm{Pr}(\mathsf{y}=-1|\mathsf{x}_1=+1,\mathsf{x}_{m+1}=+1)\right\}\\
&=w_{\{1,m+1\}}.
\end{align*}

To show $w_{\{s,s+1\}} \geq w_{\{1,m+1\}}+\gamma$, let $G'=(V,T)$ be a tree that covers $G$. Note that: 1)
\begin{align*}
\zeta_{\{s,s+1\}}(x_V)&=\zeta_{\{s,s+1\}}\Big(z_{T\setminus\{\{1,2\},\{s,s+1\}\{m,m+1\}\}},\\
&z_{\{1,2\}}=x_1x_{s}/\prod_{r=2}^{s-1}z_{\{r,r+1\}},\\
&z_{\{m,m+1\}}=x_{s+1}x_{m+1}/\prod_{r=s+1}^{m-1}z_{\{r,r+1\}}\Big),
\end{align*}
 for any $x_V\in\{+1,-1\}^d$; 2) $G'-\{\{1,2\},\{s,s+1\},\{m,m+1\}\}$ is a union of four trees where $1$, $s$, $s+1$ and $m+1$ are in different trees, such that the mapping between $x_V$ and $(z_{T\setminus\{\{1,2\},\{s,s+1\},\{m,m+1\}\}},x_1,x_s,x_{s+1},x_{m+1})$ is {\em one-on-one}. We thus have
\begin{align*}
&\sum_{x_V:(x_1,x_s,x_{s+1},x_{m+1})=(a,+1,+1,b)}\zeta_{\{s,s+1\}}(x_V)\\
&=\sum_{z_{T\setminus\{\{1,2\},\{s,s+1\},\{m,m+1\}\}}\in \{+1,-1\}^{d-4}}\zeta_{\{s,s+1\}}\Big(\\
& \hspace{14pt} z_{T\setminus\{\{1,2\},\{s,s+1\},\{m,m+1\}\}},z_{\{1,2\}}=a/\prod_{r=2}^{s-1}z_{\{r,r+1\}},\\
& \hspace{14pt}z_{\{m,m+1\}}=b/\prod_{r=s+1}^{m-1}z_{\{r,r+1\}}\Big)
\end{align*}
for any $a,b\in\{+1,-1\}$. It follows that
\begin{align*}
&\mathrm{Pr}(\mathsf{y}=+1|\mathsf{x}_s=+1,\mathsf{x}_{s+1}=+1)\\
& \hspace{40pt}-\mathrm{Pr}(\mathsf{y}=+1|\mathsf{x}_1=+1,\mathsf{x}_{m+1}=+1)\nonumber\\
&=2^{-(d-2)}\sum_{z_{T\setminus\{\{1,2\},\{s,s+1\},\{m,m+1\}\}}\in \{+1,-1\}^{d-4}}\Big[\\
& \hspace{14pt} \zeta_{\{s,s+1\}}\Big(z_{T\setminus\{\{1,2\},\{s,s+1\},\{m,m+1\}\}},\\
& \hspace{14pt} z_{\{1,2\}}=\prod_{r=2}^{s-1}z_{\{r,r+1\}},z_{\{m,m+1\}}=-\prod_{r=s+1}^{m-1}z_{\{r,r+1\}}\Big)+\\
& \hspace{14pt} \zeta_{\{s,s+1\}}\Big(z_{T\setminus\{\{1,2\},\{s,s+1\},\{m,m+1\}\}},\\
& \hspace{14pt} z_{\{1,2\}}=-\prod_{r=2}^{s-1}z_{\{r,r+1\}},z_{\{m,m+1\}}=\prod_{r=s+1}^{m-1}z_{\{r,r+1\}}\Big)\Big]
\end{align*}
and
\begin{align*}
&\mathrm{Pr}(\mathsf{y}=+1|\mathsf{x}_s=+1,\mathsf{x}_{s+1}=+1)\\
& \hspace{40pt}-\mathrm{Pr}(\mathsf{y}=-1|\mathsf{x}_1=+1,\mathsf{x}_{m+1}=+1)\nonumber\\
&=2^{-(d-2)}\sum_{z_{T\setminus\{\{1,2\},\{s,s+1\},\{m,m+1\}\}}\in \{+1,-1\}^{d-4}}\Big[\\
& \hspace{14pt} \zeta_{\{s,s+1\}}\Big(z_{T\setminus\{\{1,2\},\{s,s+1\},\{m,m+1\}\}},\\
& \hspace{14pt} z_{\{1,2\}}=\prod_{r=2}^{s-1}z_{\{r,r+1\}},z_{\{m,m+1\}}=\prod_{r=s+1}^{m-1}z_{\{r,r+1\}}\Big)+\\
& \hspace{14pt} \zeta_{\{s,s+1\}}\Big(z_{T\setminus\{\{1,2\},\{s,s+1\},\{m,m+1\}\}},\\
& \hspace{14pt} z_{\{1,2\}}=-\prod_{r=2}^{s-1}z_{\{r,r+1\}},z_{\{m,m+1\}}=-\prod_{r=s+1}^{m-1}z_{\{r,r+1\}}\Big)\Big].
\end{align*}

Assume without loss of generality that $\beta_{\{s,s+1\}}>0$. (Otherwise, we may simply replace $\beta_{\{s,s+1\}}$ by $-\beta_{\{s,s+1\}}$, and the rest of the proof remains the same.) By the monotonicity of the sigmoid function $\sigma(\cdot)$, we have $\zeta_{\{s,s+1\}}(z_{T\setminus\{\{s,s+1\}\}})> 0$ for any $z_{T\setminus\{\{s,s+1\}\}}\in\{+1,-1\}^{d-2}$. We thus have
\begin{align*}
&\mathrm{Pr}(\mathsf{y}=+1|\mathsf{x}_s=+1,\mathsf{x}_{s+1}=+1)\\
& \hspace{40pt}-\mathrm{Pr}(\mathsf{y}=+1|\mathsf{x}_1=+1,\mathsf{x}_{m+1}=+1)\nonumber\\
&\geq2^{-(d-2)}\sum_{z_{T\setminus\{\{1,2\},\{s,s+1\},\{m,m+1\}\}}\in \Delta_3}\Big[\\
& \hspace{14pt} \zeta_{\{s,s+1\}}\Big(z_{T\setminus\{\{1,2\},\{s,s+1\},\{m,m+1\}\}},\\
& \hspace{14pt} z_{\{1,2\}}=\prod_{r=2}^{s-1}z_{\{r,r+1\}},z_{\{m,m+1\}}=-\prod_{r=s+1}^{m-1}z_{\{r,r+1\}}\Big)+\\
& \hspace{14pt} \zeta_{\{s,s+1\}}\Big(z_{T\setminus\{\{1,2\},\{s,s+1\},\{m,m+1\}\}},\\
& \hspace{14pt} z_{\{1,2\}}=-\prod_{r=2}^{s-1}z_{\{r,r+1\}},z_{\{m,m+1\}}=\prod_{r=s+1}^{m-1}z_{\{r,r+1\}}\Big)\Big],
\end{align*}
where
\begin{align*}
\Delta_3&:=\left\{z_{T\setminus\{\{1,2\},\{s,s+1\},\{m,m+1\}\}}\in \{+1,-1\}^{d-4}:\right.\nonumber\\
&\hspace{14pt}\left.\left|\sum_{\{k,l\}\in T\setminus\{\{1,2\},\{s,s+1\},\{m,m+1\}\}}\beta_{\{k,l\}}z_{\{k,l\}}\right|\leq \mu\right\}.
\end{align*}
For any $z_{T\setminus\{\{1,2\},\{s,s+1\},\{m,m+1\}\}}\in \Delta_3$, we have
\begin{align*}
&\left|\beta_{\{1,2\}}\prod_{r=2}^{s-1}z_{\{r,r+1\}}+\beta_{\{m,m+1\}}\Big(-\prod_{r=s+1}^{m-1}z_{\{r,r+1\}}\Big)+\right.\\
&\hspace{50pt}\left.\sum_{\{k,l\}\in T\setminus\{\{1,2\},\{s,s+1\},\{m,m+1\}\}}\beta_{\{k,l\}}z_{\{k,l\}}\right|\\
& \leq \left|\beta_{\{1,2\}}\right|+\left|\beta_{\{m,m+1\}}\right|+\\
&\hspace{50pt}\left|\sum_{\{k,l\}\in T\setminus\{\{1,2\},\{s,s+1\},\{m,m+1\}\}}\beta_{\{k,l\}}z_{\{k,l\}}\right|\\
& \leq \mu+\mu+\mu= 3\mu,
\end{align*}
and hence
\begin{align*}
&\zeta_{\{s,s+1\}}\Big(z_{T\setminus\{\{1,2\},\{s,s+1\},\{m,m+1\}\}},\\
& \hspace{14pt} z_{\{1,2\}}=\prod_{r=2}^{s-1}z_{\{r,r+1\}},z_{\{m,m+1\}}=-\prod_{r=s+1}^{m-1}z_{\{r,r+1\}}\Big)\\
& \geq \min_{|x|\leq 3\mu}\left[\sigma\Big(\beta_{\{s,s+1\}}+x\Big)-\sigma\Big(-\beta_{\{s,s+1\}}+x\Big)\right]\\
&= \sigma\Big(\beta_{\{s,s+1\}}+3\mu\Big)-\sigma\Big(-\beta_{\{s,s+1\}}+3\mu\Big)\\
& \geq \min_{0<\beta\leq\lambda}\left[\sigma\Big(\beta+3\mu\Big)-\sigma\Big(-\beta+3\mu\Big)\right]\\
&= \sigma(\lambda+3\mu)-\sigma(-\lambda+3\mu).
\end{align*}
Similarly,
\begin{align*}
&\zeta_{\{s,s+1\}}\Big(z_{T\setminus\{\{1,2\},\{s,s+1\},\{m,m+1\}\}},\\
& \hspace{14pt} z_{\{1,2\}}=-\prod_{r=2}^{s-1}z_{\{r,r+1\}},z_{\{m,m+1\}}=\prod_{r=s+1}^{m-1}z_{\{r,r+1\}}\Big)\\
&= \sigma(\lambda+3\mu)-\sigma(-\lambda+3\mu).
\end{align*}
It follows from Lemma~\ref{lemma1} that
\begin{align*}
&\mathrm{Pr}(\mathsf{y}=+1|\mathsf{x}_s=+1,\mathsf{x}_{s+1}=+1)-\\
& \hspace{50pt} \mathrm{Pr}(\mathsf{y}=+1|\mathsf{x}_1=+1,\mathsf{x}_{m+1}=+1)\nonumber\\
&\geq \frac{|\Delta_3|}{2^{d-3}}\left[\sigma(\lambda+3\mu)-\sigma(-\lambda+3\mu)\right]\\
&\geq \frac{1}{\sqrt{2\pi d}}\left[\sigma(\lambda+3\mu)-\sigma(-\lambda+3\mu)\right]=\frac{\gamma}{2}.\\
\end{align*}

Also, by a completely analogous argument, we have
\begin{align*}
&\mathrm{Pr}(\mathsf{y}=+1|\mathsf{x}_s=+1,\mathsf{x}_{s+1}=+1)-\\
& \hspace{50pt} \mathrm{Pr}(\mathsf{y}=-1|\mathsf{x}_1=+1,\mathsf{x}_{m+1}=+1)\geq \frac{\gamma}{2},\nonumber\\
\end{align*}
and therefore,
$$w_{\{s,s+1\}}-w_{\{1,m+1\}}\geq 2\cdot \frac{\gamma}{2}=\gamma.$$

\subsection{Proof of Proposition~\ref{prop:Cond}}\label{pf:Cond}
Let us first show that $I \subseteq T$. Assume on the contrary that there exists an $\{i,j\}\in I$ but with $\{i,j\}\not\in T$. Let $\left\{\{k_1,k_2\},\{k_2,k_3\},\ldots,\{k_m,k_{m+1}\}\right\}$ be a path in $G'$ such that $k_1=i$, $k_{m+1}=j$, and $m\geq 2$. Such a path must exist since $G'$ is a spanning tree and $\{i,j\}\not\in T$ by the assumption. Let $V_1$ be the set of vertices that are connected to $i$ in $G-\{i,j\}$ and $V_2:=V-V_1$. Since $k_1\in V_1$ and $k_{m+1}\in V_2$, there must exist an $s\in\{1,\ldots,m\}$ such that $k_s\in V_1$ and $k_{s+1}\in V_2$.

If $k_s$ and $k_{s+1}$ are connected in $G$, then the unique path between $k_s$ and $k_{s+1}$ in $G$ must include $\{i,j\}$. (Otherwise, we would have $k_{s+1}\in V_1$.) By assumption 2), we have $u_{k_s,k_{s+1}}<u_{i,j}$. If $k_s$ and $k_{s+1}$ are disconnected in $G$, by assumption 1) we have $u_{k_s,k_{s+1}}<u_{i,j}$.

In either case, $G'-\{k_s,k_{s+1}\}+\{i,j\}$ is a new spanning tree with a larger total weight than $G'$, violating the assumption that $G'$ is a maximum-weight spanning tree. We therefore must have $I \subseteq T$. Furthermore, by assumption 1) we have $I \subseteq W$. We thus have $I \subseteq I\cap W$.

Now to show $I=T\cap W$, we only need to show that $|T\cap W| \leq |I|$, which can be argued as follows. Let $\omega(G)$ be the number of connected components in $G$. Then $T$ must include at least $\omega(G)-1$ edges that cross two different connected components of $G$. By assumption 1) the weights of these edges must be less than or equal to $\eta$, and we thus have
\begin{align*}
|T\cap(E-W)| \geq \omega(G)-1.
\end{align*}
It follows immediately that
\begin{align*}
|T\cap W| &= |T-T\cap(E-W)|\\
&= |T|-|T\cap(E-W)|\\
& \leq (d-1)-(\omega(G)-1)= |I|.
\end{align*}
This completes the proof of Proposition~\ref{prop:Cond}.

\section{Proof of Lemmas}
\subsection{Proof of Lemma~\ref{lemma2}}\label{pf:Lemma2}
By the assumption that $G$ is acyclic, there must exist a vertex bipartition $(V_1,V_2)$ of $G$ such that $I \subseteq \{\{k,l\}:k\in V_1,l\in V_2\}$. Assume without loss of generality that $i\in V_1$. We have
\begin{align*}
&2^{d-1}\mathrm{Pr}(\mathsf{y}=+1|\mathsf{x}_i=+1)\\
&\stackrel{(a)}{=}\sum_{x_V:x_i=+1}\mathrm{Pr}(\mathsf{y}=+1|\mathsf{x}_V=x_V)\\
&=\sum_{x_V:x_i=+1}\sigma\Big(\sum_{k\in V_1}\sum_{l\in V_2}\beta_{\{k,l\}}x_kx_l\Big)\\
&\stackrel{(b)}{=}\sum_{x_V:x_i=+1} \sigma\Big(-\sum_{k\in V_1}\sum_{l\in V_2}\beta_{\{k,l\}}x_kx_l\Big)\\
&=\sum_{x_V:x_i=+1}\mathrm{Pr}(\mathsf{y}=-1|\mathsf{x}_V=x_V)\\
&\stackrel{(c)}{=}2^{d-1}\mathrm{Pr}(\mathsf{y}=-1|\mathsf{x}_i=+1),
\end{align*}
and hence
\begin{align}
\mathrm{Pr}(\mathsf{y}=+1&|\mathsf{x}_i=+1)=\mathrm{Pr}(\mathsf{y}=-1|\mathsf{x}_i=+1).
\label{eq:W5}
\end{align}
Here, $(a)$ and $(c)$ are due to the fact that $\mathrm{Pr}\left(\mathsf{x}_{V\setminus\{i\}}=x_{V\setminus\{i\}}|\mathsf{x}_i=+1\right)=2^{-(d-1)}$ for any $x_{V\setminus\{i\}}\in \{+1,-1\}^{d-1}$, and $(b)$ follows from the change of variable $x_l \rightarrow -x_l$ for all $l \in V_2$. Since $\mathrm{Pr}(\mathsf{y}=+1|\mathsf{x}_i=+1)+\mathrm{Pr}(\mathsf{y}=-1|\mathsf{x}_i=+1)=1$, \eqref{eq:W5} immediately implies that $\mathrm{Pr}(\mathsf{y}=+1|\mathsf{x}_i=+1)=1/2$. We thus have
\begin{align*}
\mathrm{Pr}(\mathsf{x}_i=+1,\mathsf{y}=+1)\stackrel{(d)}{=}\frac{1}{2}\mathrm{Pr}(\mathsf{y}=+1|\mathsf{x}_i=+1)=\frac{1}{4},
\end{align*}
where $(d)$ follows from the fact that $\mathrm{Pr}(\mathsf{x}_i=+1)=1/2$.

\subsection{Proof of Lemma~\ref{lemma3}}\label{pf:Lemma3}
Let $G_s=(V_s,I_s)$, $s=1,\ldots,t$, be the connected components of $G-\{i,j\}$. By the assumption that $G$ is acyclic, so is $G-\{i,j\}$. Therefore, $G_s$ must be a tree for any $s=1,\ldots,t$. Further note that $i$ and $j$ are not connected in $G-\{i,j\}$, so they must belong to two different trees. Assume without loss of generality that $i\in V_1$ and $j\in V_2$. For $s=1,\ldots,t$, let $(V_s^{(1)},V_s^{(2)})$ be a vertex bipartition of $G_s$. Further assume without loss of generality that $i\in V_1^{(1)}$ and $j\in V_2^{(1)}$. Then we have
\begin{align*}
I\setminus\{\{i,j\}\} \subseteq \{\{k,l\}:k\in V^{(1)},l\in V^{(2)}\}
\end{align*}
where $V^{(1)}:=\bigcup_{s=1}^tV_s^{(1)}$ and $V^{(2)}:=\bigcup_{s=1}^tV_s^{(2)}$. Note that by construction both $i$ and $j$ are in $V^{(1)}$. It follows that
\begin{align}
&2^{d-2}\mathrm{Pr}(\mathsf{y}=+1|\mathsf{x}_i=+1,\mathsf{x}_j=+1)\nonumber\\
&=\sum_{x_V:(x_i,x_j)=(+1,+1)} \mathrm{Pr}(\mathsf{y}=+1|\mathsf{x}_V=x_V)\nonumber\\
&=\sum_{x_V:(x_i,x_j)=(+1,+1)} \zeta_{\{i,j\}}^+(x_V)\label{eq:Ste1}
\end{align}
and
\begin{align}
&2^{d-2}\mathrm{Pr}(\mathsf{y}=-1|\mathsf{x}_i=+1,\mathsf{x}_j=+1)\nonumber\\
&=\sum_{x_V:(x_i,x_j)=(+1,+1)} \mathrm{Pr}\left(\mathsf{y}=-1|\mathsf{x}_V=x_V\right)\nonumber\\
&=\sum_{x_V:(x_i,x_j)=(+1,+1)} \sigma\Big(-\beta_{\{i,j\}}-\sum_{k\in V^{(1)}}\sum_{l\in V^{(2)}}\beta_{\{k,l\}}x_kx_l\Big)\nonumber\\
&\!\stackrel{(a)}{=}\sum_{x_V:(x_i,x_j)=(+1,+1)} \sigma\Big(-\beta_{\{i,j\}}+\sum_{k\in V^{(1)}}\sum_{l\in V^{(2)}}\beta_{\{k,l\}}x_kx_l\Big)\nonumber\\
&=\sum_{x_V:(x_i,x_j)=(+1,+1)} \zeta_{\{i,j\}}^-(x_V),\label{eq:Ste2}
\end{align}
where $(a)$ follows from the change of variable $x_l \rightarrow -x_l$ for all $l \in V^{(2)}$. Combining \eqref{eq:Ste1} and \eqref{eq:Ste2} completes the proof of \eqref{eq:Ste}.

\subsection{Proof of Lemma~\ref{lemma1}}\label{pf:Lemma1}
First note that the left-hand side of \eqref{eq:lemma} is independent of the signs of the real constants $a_i$'s, and the right-hand side of \eqref{eq:lemma} is monotone decreasing with $q$. We thus may assume without loss of generality that
\begin{align*}
a_1 \geq a_2 \geq \cdots \geq a_q >0.
\end{align*}
For $i=1,\ldots,q$, let
\begin{align*}
p_i:=\mathrm{Pr}\left(-a_1 \leq a_1\sum_{j=1}^{i}\mathsf{z}_j+\frac{a_1}{a_i}\sum_{j=i+1}^{q}a_j\mathsf{z}_j < a_1\right).
\end{align*}
We have the following claim.

\begin{claim}\label{clm}
$p_i$ is monotone decreasing with $i$.
\end{claim}

\begin{proof}
Fix $i\in\{1,\ldots,q-1\}$. We have
\begin{align*}
p_{i}&=\mathrm{Pr}\left(-a_1 \leq a_1\sum_{j=1}^{i}\mathsf{z}_j+\frac{a_1}{a_i}\sum_{j=i+1}^{q}a_j\mathsf{z}_j < a_1\right)\\
&=2^{-i}\sum_{(z_1,\ldots,z_{i})\in\{+1,-1\}^{i}}\mathrm{Pr}\left(-\left(1+\sum_{j=1}^{i}z_j\right)a_1\right.\\
&\hspace{60pt} \left. \leq \frac{a_1}{a_i}\sum_{j=i+1}^{q}a_j\mathsf{z}_j < \left(1-\sum_{j=1}^{i}z_j\right)a_1\right)\\
&=2^{-i}\sum_{k=0}^{\lfloor i/2\rfloor}c_{i,k}\\
&\cdot\mathrm{Pr}\left(-\left(1+i-2k\right)a_1\leq \frac{a_1}{a_{i}}\sum_{j=i+1}^{q}a_j\mathsf{z}_j < \left(1+i-2k\right)a_1\right)\\
&\stackrel{(a)}{\geq}2^{-i}\sum_{k=0}^{\lfloor i/2\rfloor}c_{i,k}\\
&\cdot\mathrm{Pr}\left(-\left(1+i-2k\right)a_1\leq \frac{a_1}{a_{i+1}}\sum_{j=i+1}^{q}a_j\mathsf{z}_j < \left(1+i-2k\right)a_1\right)\\
&=2^{-i}\sum_{(z_1,\ldots,z_{i})\in\{+1,-1\}^{i}}\mathrm{Pr}\left(-\left(1+\sum_{j=1}^{i}z_j\right)a_1\right.\\
&\hspace{60pt} \left. \leq \frac{a_1}{a_{i+1}}\sum_{j=i+1}^{q}a_j\mathsf{z}_j < \left(1-\sum_{j=1}^{i}z_j\right)a_1\right)\\
&=\mathrm{Pr}\left(-a_1 \leq a_1\sum_{j=1}^{i}\mathsf{z}_j+\frac{a_1}{a_{i+1}}\sum_{j=i+1}^{q}a_j\mathsf{z}_j < a_1\right)\\
&=\mathrm{Pr}\left(-a_1 \leq a_1\sum_{j=1}^{i+1}\mathsf{z}_j+\frac{a_1}{a_{i+1}}\sum_{j=i+2}^{q}a_j\mathsf{z}_j < a_1\right)\\
&=p_{i+1},
\end{align*}
where
\begin{align*}
c_{i,k}:=\left\{
\begin{array}{rl}
1, & \mbox{if $k=0$},\\
\textstyle \binom{i}{k}-\binom{i}{k-1}, & \mbox{if $1 \leq k \leq \lfloor i/2\rfloor$},
\end{array}
\right.
\end{align*}
and $(a)$ follows from the assumption that $a_i\geq a_{i+1}>0$. We may thus conclude that $p_i$ is monotone decreasing with $i$.
\end{proof}

By Claim~\ref{clm}, we have
\begin{align*}
&\mathrm{Pr}\left(\left|\sum_{j=1}^{q}a_j\mathsf{z}_j\right|\geq a_1\right)\geq p_1\geq p_q\\
&=\mathrm{Pr}\left(-1 \leq \sum_{j=1}^{q}\mathsf{z}_j < 1\right)=\frac{1}{2^q}
\binom{q}{\lfloor q/2\rfloor}\stackrel{(b)}{\geq} \sqrt{\frac{2}{\pi(q+2)}},
\end{align*}
where $(b)$ follows from the well-known Wallis' product \cite{be2004} for $\pi$. This completes the proof of Lemma~\ref{lemma1}.

\subsection{Proof of Lemmas~\ref{lemma4} and \ref{lemma5}}\label{pf:Lemma4-5}
Let $G_r=(V_r,I_r)$, $r=1,\ldots,t$, be the connected components of $G-\{\{1,2\},\{2,3\},\ldots,\{m,m+1\}\}$. By the assumption that $G$ is acyclic, so is $G-\{\{1,2\},\{2,3\},\ldots,\{m,m+1\}\}$. Therefore, each connected component $G_r$ must be a tree, and each of the vertices from $\{1,\ldots,m+1\}$ is in a different tree. Further assume without loss of generality that $r\in V_r$ for all $r=1,\ldots,t$ (where apparently we have $t \geq m+1$). Let $V_r^{odd}$ and $V_r^{even}$ be the sets of vertices from $G_r$ whose (graphical) distances to $r$ are odd and even, respectively.

To prove Lemma~\ref{lemma4}, define for any $a,b,c\in\{+1,-1\}$
\begin{align}
&p^+_{1,2,m+1}(a,b,c)\nonumber\\
&:=\sum_{x_V:(x_1,x_2,x_{m+1})=(a,b,c)} \mathrm{Pr}\left(\mathsf{y}=+1|\mathsf{x}_V=x_V\right)\nonumber\\
&=\sum_{x_V:(x_1,x_2,x_{m+1})=(a,b,c)} \sigma\Big(\sum_{\{k,l\}\in I}\beta_{\{k,l\}}x_kx_l\Big)\nonumber\end{align}
and
\begin{align}
&p^-_{1,2,m+1}(a,b,c)\nonumber\\
&:=\sum_{x_V:(x_1,x_2,x_{m+1})=(a,b,c)} \mathrm{Pr}\left(\mathsf{y}=-1|\mathsf{x}_V=x_V\right)\nonumber\\
&=\sum_{x_V:(x_1,x_2,x_{m+1})=(a,b,c)} \sigma\Big(-\sum_{\{k,l\}\in I}\beta_{\{k,l\}}x_kx_l\Big).\nonumber
\end{align}
We can write
\begin{align}
&2^{d-2}\mathrm{Pr}(\mathsf{y}=+1|\mathsf{x}_1=+1,\mathsf{x}_{2}=+1)\nonumber\\
&\hspace{13pt}=p^+_{1,2,m+1}(+1,+1,+1)+p^+_{1,2,m+1}(+1,+1,-1),\label{eq:Bos3}\\
&2^{d-2}\mathrm{Pr}(\mathsf{y}=+1|\mathsf{x}_1=+1,\mathsf{x}_{m+1}=+1)\nonumber\\
&\hspace{13pt}=p^+_{1,2,m+1}(+1,+1,+1)+p^+_{1,2,m+1}(+1,-1,+1),\label{eq:Bos4}\\
&2^{d-2}\mathrm{Pr}(\mathsf{y}=-1|\mathsf{x}_1=+1,\mathsf{x}_{m+1}=+1)\nonumber\\
&\hspace{13pt}=p^-_{1,2,m+1}(+1,+1,+1)+p^-_{1,2,m+1}(+1,-1,+1).\label{eq:Bos5}
\end{align}

For any $m \geq 2$, we have
\begin{align}
&p^+_{1,2,m+1}(+1,+1,+1)\nonumber\\
&=\sum_{x_V:(x_1,x_2,x_{m+1})=(+1,+1,+1)}\zeta_{\{1,2\}}^+(x_V),\label{eq:Bos6}\\
&p^+_{1,2,m+1}(+1,+1,-1)\nonumber\\
&=\sum_{x_V:(x_1,x_2,x_{m+1})=(+1,+1,-1)}\zeta_{\{1,2\}}^+(x_V),\label{eq:Bos7}\\
&p^+_{1,2,m+1}(+1,-1,+1)\nonumber\\
&=\sum_{x_V:(x_1,x_2,x_{m+1})=(+1,-1,+1)}\sigma\Big(\sum_{\{k,l\}\in I}\beta_{\{k,l\}}x_kx_l\Big)\nonumber\\
&\stackrel{(a)}{=}\sum_{x_V:(x_1,x_2,x_{m+1})=(+1,+1,-1)}\sigma\Big(-\beta_{\{1,2\}}+\nonumber\\
& \hspace{100pt}\sum_{\{k,l\}\in I\setminus\{\{1,2\}\}}\beta_{\{k,l\}}x_kx_l\Big)\nonumber\\
&=\sum_{x_V:(x_1,x_2,x_{m+1})=(+1,+1,-1)}\zeta_{\{1,2\}}^-(x_V),\label{eq:Bos8}
\end{align}
where $(a)$ follows from the change of variable $x_k \rightarrow -x_k$ for all $k \in \bigcup_{r=2}^{m+1}V_r$. Substituting \eqref{eq:Bos7} and \eqref{eq:Bos8} into \eqref{eq:Bos3} and \eqref{eq:Bos4} completes the proof of \eqref{eq:Bos1}.

For any $m\geq 2$ that is odd, we have
\begin{align}
&p^-_{1,2,m+1}(+1,+1,+1)\nonumber\\
&=\sum_{x_V:(x_1,x_2,x_{m+1})=(+1,+1,+1)}\sigma\Big(-\sum_{\{k,l\}\in I}\beta_{\{k,l\}}x_kx_l\Big)\nonumber\\
&\stackrel{(b)}{=}\sum_{x_V:(x_1,x_2,x_{m+1})=(+1,+1,+1)}\sigma\Big(-\beta_{\{1,2\}}+\nonumber\\
& \hspace{100pt}\sum_{\{k,l\}\in I\setminus\{\{1,2\}\}}\beta_{\{k,l\}}x_kx_l\Big)\nonumber\\
&=\sum_{x_V:(x_1,x_2,x_{m+1})=(+1,+1,+1)}\zeta_{\{1,2\}}^-(x_V),\label{eq:Bos9}\\
&p^-_{1,2,m+1}(+1,-1,+1)\nonumber\\
&=\sum_{x_V:(x_1,x_2,x_{m+1})=(+1,-1,+1)}\sigma\Big(-\sum_{\{k,l\}\in I}\beta_{\{k,l\}}x_kx_l\Big)\nonumber\\
&\stackrel{(c)}{=}\sum_{x_V:(x_1,x_2,x_{m+1})=(+1,+1,-1)}\sigma\Big(\beta_{\{1,2\}}+\nonumber\\
& \hspace{100pt}\sum_{\{k,l\}\in I\setminus\{\{1,2\}\}}\beta_{\{k,l\}}x_kx_l\Big)\nonumber\\
&=\sum_{x_V:(x_1,x_2,x_{m+1})=(+1,+1,-1)}\zeta_{\{1,2\}}^+(x_V),\label{eq:Bos10}
\end{align}
where $(b)$ follows from the change of variable $x_k \rightarrow -x_k$ for all $k \in V_r^{odd}$ for $r \in \{1,2,4,\ldots,m+1,m+2,\ldots,t\}$ or $k \in V_r^{even}$ for $r \in \{3,5,\ldots,m\}$, and $(c)$ follows from the change of variable $x_k \rightarrow -x_k$ for all $k \in V_r^{odd}$ for $r \in \{1,3,\ldots,m\}$ or $k \in V_r^{even}$ for $r \in \{2,4,\ldots,m+1,m+2,\ldots,t\}$. Substituting \eqref{eq:Bos6}, \eqref{eq:Bos7}, \eqref{eq:Bos9} and \eqref{eq:Bos10} into \eqref{eq:Bos4} and \eqref{eq:Bos5} completes the proof of \eqref{eq:Bos2}.

To prove Lemma~\ref{lemma5}, let us first write
\begin{align}
&2^{d-2}\mathrm{Pr}(\mathsf{y}=+1|\mathsf{x}_s=+1,\mathsf{x}_{s+1}=+1)\nonumber\\
&\hspace{14pt}=p^+_{1,s,s+1,m+1}(+1,+1,+1,+1)+\nonumber\\
&\hspace{26pt}p^+_{1,s,s+1,m+1}(+1,+1,+1,-1)+\nonumber\\
&\hspace{26pt}p^+_{1,s,s+1,m+1}(-1,+1,+1,+1)+\nonumber\\
&\hspace{26pt}p^+_{1,s,s+1,m+1}(-1,+1,+1,-1),\label{eq:TB3}\\
&2^{d-2}\mathrm{Pr}(\mathsf{y}=+1|\mathsf{x}_1=+1,\mathsf{x}_{m+1}=+1)\nonumber\\
&\hspace{14pt}=p^+_{1,s,s+1,m+1}(+1,+1,+1,+1)+\nonumber\\
&\hspace{26pt}p^+_{1,s,s+1,m+1}(+1,+1,-1,+1)+\nonumber\\
&\hspace{26pt}p^+_{1,s,s+1,m+1}(+1,-1,+1,+1)+\nonumber\\
&\hspace{26pt}p^+_{1,s,s+1,m+1}(+1,-1,-1,+1),\label{eq:TB4}\\
&2^{d-2}\mathrm{Pr}(\mathsf{y}=-1|\mathsf{x}_1=+1,\mathsf{x}_{m+1}=+1)\nonumber\\
&\hspace{14pt}=p^-_{1,s,s+1,m+1}(+1,+1,+1,+1)+\nonumber\\
&\hspace{26pt}p^-_{1,s,s+1,m+1}(+1,+1,-1,+1)+\nonumber\\
&\hspace{26pt}p^-_{1,s,s+1,m+1}(+1,-1,+1,+1)+\nonumber\\
&\hspace{26pt}p^-_{1,s,s+1,m+1}(+1,-1,-1,+1).\label{eq:TB5}
\end{align}


For any $m \geq 3$, we have
\begin{align}
&p^+_{1,s,s+1,m+1}(+1,+1,+1,+1)\nonumber\\
&=\sum_{x_V:(x_1,x_s,x_{s+1},x_{m+1})=(+1,+1,+1,+1)}\zeta_{\{s,s+1\}}^+(x_V),\label{eq:TB6}\\
&p^+_{1,s,s+1,m+1}(+1,+1,+1,-1)\nonumber\\
&=\sum_{x_V:(x_1,x_s,x_{s+1},x_{m+1})=(+1,+1,+1,-1)}\zeta_{\{s,s+1\}}^+(x_V),\label{eq:TB7}\\
&p^+_{1,s,s+1,m+1}(-1,+1,+1,+1)\nonumber\\
&=\sum_{x_V:(x_1,x_s,x_{s+1},x_{m+1})=(-1,+1,+1,+1)}\zeta_{\{s,s+1\}}^+(x_V),\label{eq:TB8}\\
&p^+_{1,s,s+1,m+1}(-1,+1,+1,-1)\nonumber\\
&=\sum_{x_V:(x_1,x_s,x_{s+1},x_{m+1})=(-1,+1,+1,-1)}\zeta_{\{s,s+1\}}^+(x_V),\label{eq:TB9}\\
&p^+_{1,s,s+1,m+1}(+1,+1,-1,+1)\nonumber\\
&=\sum_{x_V:(x_1,x_s,x_{s+1},x_{m+1})=(+1,+1,-1,+1)}\hspace{-0.5mm}\sigma\Big(\hspace{-0.5mm}\sum_{\{k,l\}\in I}\beta_{\{k,l\}}x_kx_l\Big)\nonumber\\
&\stackrel{(d)}{=}\sum_{x_V:(x_1,x_s,x_{s+1},x_{m+1})=(+1,+1,+1,-1)}\sigma\Big(-\beta_{\{s,s+1\}}+\nonumber\\
& \hspace{50pt} \sum_{\{k,l\}\in I\setminus\{\{s,s+1\}\}}\beta_{\{k,l\}}x_kx_l\Big)\nonumber\\
&=\sum_{x_V:(x_1,x_s,x_{s+1},x_{m+1})=(+1,+1,+1,-1)}\zeta_{\{s,s+1\}}^-(x_V),\label{eq:TB10}\\
&p^+_{1,s,s+1,m+1}(+1,-1,+1,+1)\nonumber\\
&=\sum_{x_V:(x_1,x_s,x_{s+1},x_{m+1})=(+1,-1,+1,+1)}\hspace{-0.5mm}\sigma\Big(\hspace{-0.5mm}\sum_{\{k,l\}\in I}\beta_{\{k,l\}}x_kx_l\Big)\nonumber\\
&\stackrel{(e)}{=}\sum_{x_V:(x_1,x_s,x_{s+1},x_{m+1})=(-1,+1,+1,+1)}\sigma\Big(-\beta_{\{s,s+1\}}+\nonumber\\
& \hspace{50pt} \sum_{\{k,l\}\in I\setminus\{\{s,s+1\}\}}\beta_{\{k,l\}}x_kx_l\Big)\nonumber\\
&=\sum_{x_V:(x_1,x_s,x_{s+1},x_{m+1})=(-1,+1,+1,+1)}\zeta_{\{s,s+1\}}^-(x_V),\label{eq:TB11}\\
&p^+_{1,s,s+1,m+1}(+1,-1,-1,+1)\nonumber\\
&=\sum_{x_V:(x_1,x_s,x_{s+1},x_{m+1})=(+1,-1,-1,+1)}\hspace{-0.5mm}\sigma\Big(\hspace{-0.5mm}\sum_{\{k,l\}\in I}\beta_{\{k,l\}}x_kx_l\Big)\nonumber\\
&\stackrel{(f)}{=}\sum_{x_V:(x_1,x_s,x_{s+1},x_{m+1})=(-1,+1,+1,-1)}\sigma\Big(\beta_{\{s,s+1\}}+\nonumber\\
& \hspace{50pt} \sum_{\{k,l\}\in I\setminus\{\{s,s+1\}\}}\beta_{\{k,l\}}x_kx_l\Big)\nonumber\\
&=\sum_{x_V:(x_1,x_s,x_{s+1},x_{m+1})=(-1,+1,+1,-1)}\zeta_{\{s,s+1\}}^+(x_V),\label{eq:TB12}
\end{align}
where $(d)$ follows from the change of variable $x_k \rightarrow -x_k$ for all $k\in \bigcup_{r=s+1}^{m+1}V_r$, $(e)$ follows from the change of variable $x_k \rightarrow -x_k$ for all $k\in \bigcup_{r=1}^{s}V_r$, and $(f)$ follows from the change of variable $x_k \rightarrow -x_k$ for all $k\in \bigcup_{r=1}^{m+1}V_r$. Substituting \eqref{eq:TB7}--\eqref{eq:TB12} into \eqref{eq:TB3} and \eqref{eq:TB4} completes the proof of \eqref{eq:TB1}.

To prove \eqref{eq:TB2}, assume that $m \geq 3$ is odd. When $s$ is even, we have
\begin{align}
&p^-_{1,s,s+1,m+1}(+1,+1,+1,+1)\nonumber\\
&=\sum_{x_V:(x_1,x_s,x_{s+1},x_{m+1})=(+1,+1,+1,+1)}\nonumber\\
& \hspace{24pt} \sigma\Big(-\sum_{\{k,l\}\in I}\beta_{\{k,l\}}x_kx_l\Big)\nonumber\\
&\stackrel{(g)}{=}\sum_{x_V:(x_1,x_s,x_{s+1},x_{m+1})=(-1,+1,+1,-1)}\nonumber\\
& \hspace{24pt} \sigma\Big(-\beta_{\{s,s+1\}}+\sum_{\{k,l\}\in I\setminus\{\{s,s+1\}\}}\beta_{\{k,l\}}x_kx_l\Big)\nonumber\\
&=\sum_{x_V:(x_1,x_s,x_{s+1},x_{m+1})=(-1,+1,+1,-1)}\zeta_{\{s,s+1\}}^-(x_V),\label{eq:TB13}\\
&p^-_{1,s,s+1,m+1}(+1,+1,-1,+1)\nonumber\\
&=\sum_{x_V:(x_1,x_s,x_{s+1},x_{m+1})=(+1,+1,-1,+1)}\nonumber\\
& \hspace{24pt} \sigma\Big(-\sum_{\{k,l\}\in I}\beta_{\{k,l\}}x_kx_l\Big)\nonumber\\
&\stackrel{(h)}{=}\sum_{x_V:(x_1,x_s,x_{s+1},x_{m+1})=(-1,+1,+1,+1)}\nonumber\\
& \hspace{24pt} \sigma\Big(\beta_{\{s,s+1\}}+\sum_{\{k,l\}\in I\setminus\{\{s,s+1\}\}}\beta_{\{k,l\}}x_kx_l\Big)\nonumber\\
&=\sum_{x_V:(x_1,x_s,x_{s+1},x_{m+1})=(-1,+1,+1,+1)}\zeta_{\{s,s+1\}}^+(x_V),\label{eq:TB14}\\
&p^-_{1,s,s+1,m+1}(+1,-1,+1,+1)\nonumber\\
&=\sum_{x_V:(x_1,x_s,x_{s+1},x_{m+1})=(+1,-1,+1,+1)}\nonumber\\
& \hspace{24pt} \sigma\Big(-\sum_{\{k,l\}\in I}\beta_{\{k,l\}}x_kx_l\Big)\nonumber\\
&\stackrel{(i)}{=}\sum_{x_V:(x_1,x_s,x_{s+1},x_{m+1})=(+1,+1,+1,-1)}\nonumber\\
& \hspace{24pt} \sigma\Big(\beta_{\{s,s+1\}}+\sum_{\{k,l\}\in I\setminus\{\{s,s+1\}\}}\beta_{\{k,l\}}x_kx_l\Big)\nonumber\\
&=\sum_{x_V:(x_1,x_s,x_{s+1},x_{m+1})=(+1,+1,+1,-1)}\zeta_{\{s,s+1\}}^+(x_V),\label{eq:TB15}\\
&p^-_{1,s,s+1,m+1}(+1,-1,-1,+1)\nonumber\\
&=\sum_{x_V:(x_1,x_s,x_{s+1},x_{m+1})=(+1,-1,-1,+1)}\nonumber\\
& \hspace{24pt} \sigma\Big(-\sum_{\{k,l\}\in I}\beta_{\{k,l\}}x_kx_l\Big)\nonumber\\
&\stackrel{(j)}{=}\sum_{x_V:(x_1,x_s,x_{s+1},x_{m+1})=(+1,+1,+1,+1)}\nonumber\\
& \hspace{24pt} \sigma\Big(-\beta_{\{s,s+1\}}+\sum_{\{k,l\}\in I\setminus\{\{s,s+1\}\}}\beta_{\{k,l\}}x_kx_l\Big)\nonumber\\
&=\sum_{x_V:(x_1,x_s,x_{s+1},x_{m+1})=(+1,+1,+1,+1)}\zeta_{\{s,s+1\}}^-(x_V),\label{eq:TB16}
\end{align}
where $(g)$ follows from the change of variable $x_k \rightarrow -x_k$ for all $k \in V_r^{odd}$ for $r \in \{2,4,\ldots,s,s+1,s+3,\ldots,m\}$ or $k \in V_r^{even}$ for $r \in \{1,3,\ldots,s-1,s+2,s+4,\ldots,m+1,m+2,\ldots,t\}$, $(h)$ follows from the change of variable $x_k \rightarrow -x_k$ for all $k \in V_r^{odd}$ for $r \in \{2,4,\ldots,m+1,m+2,\ldots,t\}$ or $k \in V_r^{even}$ for $r \in \{1,3,\ldots,m\}$, $(i)$ follows from the change of variable $x_k \rightarrow -x_k$ for all $k \in V_r^{odd}$ for $r \in\{1,3,\ldots,m\}$ or $k \in V_r^{even}$ for $r \in\{2,4,\ldots,m+1,m+2,\ldots,t\}$, and $(j)$ follows from the change of variable $x_k \rightarrow -x_k$ for all $k \in V_r^{odd}$ for $r\in\{1,3,\ldots,s-1,s+2,s+4,\ldots,m+1,m+2,\ldots,t\}$ or $k \in V_r^{even}$ for $r\in\{2,4,\ldots,s,s+1,s+3\ldots,m\}$. Substituting \eqref{eq:TB6}-\eqref{eq:TB9} and \eqref{eq:TB13}-\eqref{eq:TB16} into \eqref{eq:TB3} and \eqref{eq:TB5} completes the proof of \eqref{eq:TB2} when $s$ is even.

When $s$ is odd, we have
\begin{align}
&p^-_{1,s,s+1,m+1}(+1,+1,+1,+1)\nonumber\\
&=\sum_{x_V:(x_1,x_s,x_{s+1},x_{m+1})=(+1,+1,+1,+1)}\nonumber\\
& \hspace{24pt} \sigma\Big(-\sum_{\{k,l\}\in I}\beta_{\{k,l\}}x_kx_l\Big)\nonumber\\
&\stackrel{(k)}{=}\sum_{x_V:(x_1,x_s,x_{s+1},x_{m+1})=(+1,+1,+1,+1)}\nonumber\\
& \hspace{24pt} \sigma\Big(-\beta_{\{s,s+1\}}+\sum_{\{k,l\}\in I\setminus\{\{s,s+1\}\}}\beta_{\{k,l\}}x_kx_l\Big)\nonumber\\
&=\sum_{x_V:(x_1,x_s,x_{s+1},x_{m+1})=(-1,+1,+1,-1)}\zeta_{\{s,s+1\}}^-(x_V),\label{eq:TB17}\\
&p^-_{1,s,s+1,m+1}(+1,+1,-1,+1)\nonumber\\
&=\sum_{x_V:(x_1,x_s,x_{s+1},x_{m+1})=(+1,+1,-1,+1)}\nonumber\\
& \hspace{24pt} \sigma\Big(-\sum_{\{k,l\}\in I}\beta_{\{k,l\}}x_kx_l\Big)\nonumber\\
&\stackrel{(l)}{=}\sum_{x_V:(x_1,x_s,x_{s+1},x_{m+1})=(+1,+1,+1,-1)}\nonumber\\
& \hspace{24pt} \sigma\Big(\beta_{\{s,s+1\}}+\sum_{\{k,l\}\in I\setminus\{\{s,s+1\}\}}\beta_{\{k,l\}}x_kx_l\Big)\nonumber\\
&=\sum_{x_V:(x_1,x_s,x_{s+1},x_{m+1})=(-1,+1,+1,-1)}\zeta_{\{s,s+1\}}^+(x_V),\label{eq:TB18}\\
&p^-_{1,s,s+1,m+1}(+1,-1,+1,+1)\nonumber\\
&=\sum_{x_V:(x_1,x_s,x_{s+1},x_{m+1})=(+1,-1,+1,+1)}\nonumber\\
& \hspace{24pt} \sigma\Big(-\sum_{\{k,l\}\in I}\beta_{\{k,l\}}x_kx_l\Big)\nonumber\\
&\stackrel{(m)}{=}\sum_{x_V:(x_1,x_s,x_{s+1},x_{m+1})=(-1,+1,+1,+1)}\nonumber\\
& \hspace{24pt} \sigma\Big(\beta_{\{s,s+1\}}+\sum_{\{k,l\}\in I\setminus\{\{s,s+1\}\}}\beta_{\{k,l\}}x_kx_l\Big)\nonumber\\
&=\sum_{x_V:(x_1,x_s,x_{s+1},x_{m+1})=(-1,+1,+1,+1)}\zeta_{\{s,s+1\}}^+(x_V),\label{eq:TB19}\\
&p^-_{1,s,s+1,m+1}(+1,-1,-1,+1)\nonumber\\
&=\sum_{x_V:(x_1,x_s,x_{s+1},x_{m+1})=(+1,-1,-1,+1)}\nonumber\\
& \hspace{24pt} \sigma\Big(-\sum_{\{k,l\}\in I}\beta_{\{k,l\}}x_kx_l\Big)\nonumber\\
&\stackrel{(n)}{=}\sum_{x_V:(x_1,x_s,x_{s+1},x_{m+1})=(-1,+1,+1,-1)}\nonumber\\
& \hspace{24pt} \sigma\Big(-\beta_{\{s,s+1\}}+\sum_{\{k,l\}\in I\setminus\{\{s,s+1\}\}}\beta_{\{k,l\}}x_kx_l\Big)\nonumber\\
&=\sum_{x_V:(x_1,x_s,x_{s+1},x_{m+1})=(-1,+1,+1,-1)}\zeta_{\{s,s+1\}}^-(x_V),\label{eq:TB20}
\end{align}
where $(k)$ follows from the change of variable $x_k \rightarrow -x_k$ for all $k \in V_r^{odd}$ for $r \in \{1,3,\ldots,s,s+1,s+3,\ldots,m+1,m+2,\ldots,t\}$ or $k \in V_r^{even}$ for $r \in \{2,4,\ldots,s-1,s+2,s+4,\ldots,m\}$, $(l)$ follows from the change of variable $x_k \rightarrow -x_k$ for all $k \in V_r^{odd}$ for $r \in \{1,3,\ldots,m\}$ or $k \in V_r^{even}$ for $r \in\{2,4,\ldots,m+1,m+2,\ldots,t\}$, $(m)$ follows from the change of variable $x_k \rightarrow -x_k$ for all $k \in V_r^{odd}$ for $r \in\{2,4,\ldots,m+1,m+2,\ldots,t\}$ or $k \in V_r^{even}$ for $r \in\{1,3,\ldots,m\}$, and $(n)$ follows from the change of variable $x_k \rightarrow -x_k$ for all $k \in V_r^{odd}$ for $r\in\{2,4,\ldots,s-1,s+2,s+4,\ldots,m\}$ or $k \in V_r^{even}$ for $r\in\{1,3,\ldots,s,s+1,s+3,\ldots,m+1,m+2,\ldots,t\}$. Substituting \eqref{eq:TB6}-\eqref{eq:TB9} and \eqref{eq:TB17}-\eqref{eq:TB20} into \eqref{eq:TB3} and \eqref{eq:TB5} completes the proof of \eqref{eq:TB2} when $s$ is odd.

\bibliographystyle{ieeetr}

%
%
%



%

\end{document}